\documentclass[twoside]{article}
\pdfoutput=1
\usepackage{tikzsymbols}
\usepackage[accepted]{aistats2024}
\usepackage{fancyhdr}

\usepackage[round]{natbib}
\let\cite=\citep

\usepackage{amsmath}
\usepackage{amssymb}
\usepackage{mathtools}
\usepackage{amsthm}
\usepackage{tabularx}
\usepackage{multirow}
\usepackage{array}
\usepackage{makecell}
\usepackage{arydshln}
\usepackage{amssymb}
\usepackage{pifont}
\usepackage{bbding}
\newcommand{\mathbbm}[1]{\text{\usefont{U}{bbm}{m}{n}#1}}
\usepackage{hyperref}
\usepackage{subcaption}

\newtheorem{definition}{Definition}
\newtheorem{theorem}{Theorem}

\usepackage{algorithm}
\usepackage{algorithmic}
\usepackage{tikz}
\usetikzlibrary{arrows}

\hypersetup{
  colorlinks   = true, %Colours links instead of ugly boxes
  urlcolor     = cyan, %Colour for external hyperlinks
  linkcolor    = blue, %Colour of internal links
  citecolor    = blue, %Colour of citations
}

\newcommand{\Eq}[1]{Eq.~(\ref{eq:#1})}
\newcommand{\norm}[1]{\left\lVert#1\right\rVert}
\newcommand{\Dcal}[0]{\mathcal{D}}

\newcommand{\Mcal}[0]{\mathcal{M}}

\newcommand{\Scal}[0]{\mathcal{S}}
\usepackage{amsmath,amsfonts,bm}

\def\eqref#1{equation~\ref{#1}}
\def\1{\bm{1}}

\def\vtheta{{\bm{\theta}}}

\def\vphi{{\bm{\phi}}}

\def\vx{{\bm{x}}}
\def\vy{{\bm{y}}}

\def\mI{{\bm{I}}}

\def\mX{{\bm{X}}}

\DeclareMathAlphabet{\mathsfit}{\encodingdefault}{\sfdefault}{m}{sl}
\SetMathAlphabet{\mathsfit}{bold}{\encodingdefault}{\sfdefault}{bx}{n}

\DeclareMathOperator*{\argmax}{arg\,max}

\begin{document}
\runningtitle{SIFU: Efficient and Provable Client Unlearning in Federated Optimization}

\runningauthor{Fraboni, Van Waerebeke, Vidal, Kameni, Scaman, Lorenzi}

\twocolumn[

\aistatstitle{SIFU: Sequential Informed Federated Unlearning
for\\ Efficient and Provable Client Unlearning in Federated Optimization}

% \aistatsauthor{Yann Fraboni$^{1,\textbf{*}}$\quad Martin Van Waerebeke\footnote{Equal contribution} \quad Richard Vidal\\\\}
% \aistatsauthor{Laëtitia Kameni \quad Kevin Scaman \quad Marco Lorenzi}

% \aistatsaddress{\\
% $^1$ Accenture Labs\\
% $^2$ INRIA Paris\\
% $^3$ INRIA Sophia-Antipolis}
% ]

\aistatsauthor{Yann Fraboni$^\textbf{*}$ \And Martin Van Waerebeke$^\textbf{*}$ \And Richard Vidal}

\aistatsaddress{Accenture Labs \\ INRIA Sophia-Antipolis \And INRIA Paris \\ DIENS - PSL \And Accenture Labs} 

\aistatsauthor{Laetitia Kameni \And Kevin Scaman \And Marco Lorenzi}

\aistatsaddress{Accenture Labs \And INRIA Paris \\ DIENS - PSL \And INRIA Sophia-Antipolis}
]

\begin{abstract}
Machine Unlearning (MU) is an increasingly important topic in machine learning safety, aiming at removing the contribution of a given data point from a training procedure. Federated Unlearning (FU) consists in extending MU to unlearn a given client’s contribution from a federated training routine. While several FU methods have been proposed, we currently lack a general approach providing formal unlearning guarantees to the \textsc{FedAvg} routine, while ensuring scalability and generalization beyond the convex assumption on the clients’ loss functions. We aim at filling this gap by proposing SIFU (Sequential Informed Federated Unlearning), a new FU method applying to both convex and non-convex optimization regimes. SIFU naturally applies to \textsc{FedAvg} without additional computational cost for the clients and provides formal guarantees on the quality of the unlearning task. We provide a theoretical analysis of the unlearning properties of SIFU, and practically demonstrate its effectiveness as compared to a panel of unlearning methods from the state-of-the-art.
\end{abstract}

\section{Introduction}

With the emergence of new data regulations, such as the EU General Data Protection Regulation (GDPR) \cite{GDPR} and the California Consumer Privacy Act (CCPA) \cite{CCPA}, the storage and processing of sensitive personal data is often the subject of strict constraints and restrictions. In particular, the “right to be forgotten” states that personal data must be erased upon request from the concerned individuals, with subsequent potential implications on machine learning models trained by using this data.  
Machine Unlearning (MU) is an emerging discipline that studies methods that aim at removing the contribution of given data instances used to train a machine learning model.
Current MU approaches are essentially based on routines that modify the model weights in order to guarantee the “unlearning" of a given data point, i.e. to obtain a model equivalent to a hypothetical one trained  without this data point \cite{Cao2015TowardsMS,SISA}. 

Motivated by data governance and confidentiality concerns, 
Federated Learning (FL) has gained popularity in the last years to allow data owners to collaboratively learn a model without sharing their respective data. 
Among the different FL approaches, federated averaging (\textsc{FedAvg}) has emerged as the most popular optimization scheme \cite{FedAvg}. An optimization round of \textsc{FedAvg} requires data owners to receive the current global model from the server, which is updated 
by performing a fixed amount of Stochastic Gradient Descent (SGD) steps 
before sending back the resulting model. The new global model is then created as the weighted average of the client updates. 
The FL communication design ensures clients that their data is solely used to compute their model update, while established theory guarantees convergence of the model to a stationary point of the clients' joint optimization problem \cite{FedNova,OnTheConvergence}.

With the current deployments of FL in the real-world, it is of crucial importance to extend MU to Federated Unlearning (FU), for guaranteeing the unlearning of clients wishing to opt-out from a collaborative training routine. 
% Overall, federated unlearning (FU) schemes must be developed to unlearn a client's contribution from a FL routine without any additional computation and communication on the client side. 
This is not straightforward, since most current MU schemes have been proposed in the centralized learning setting, and cannot be seamlessly applied to the federated one. Typical issues include the need for exchanging high-order quantities related to the model parameters, or additional and potentially sensitive client information \cite{CertifiedDataRemoval,ApproximateDataDeletion,Golatkar_2020_CVPR,golatkar2020forgetting,Golatkar_2021_CVPR}.
%Exact unlearning  methods tend to fail because having to unlearn an entire client makes slicing impossible \cite{MachineUnlearning}. In approximate MU, a popular paradigm is using the quasi-Newton method by estimating the Hessian of the loss function \cite{CertifiedDataRemoval,ApproximateDataDeletion,Golatkar_2020_CVPR,golatkar2020forgetting,Golatkar_2021_CVPR}, which is not only a costly operation clients cannot necessarily perform, but also a potential privacy threat against attacks such as model inversion \cite{Fredrikson-MI-2015}. Another promising path is differential privacy-inspired \cite{DP_book} forgetting. It is certified by introducing noise tailored to clients' data, which are inaccessible by design in FL.
While several Federated Unlearning (FU) methods have been proposed, few are backed by theoretical guarantees on the effectiveness of unlearning \cite{liu2022fedun, jin2023ntk}, or are compatible with typical FL assumptions on data access or availability \cite{FedEraser}.

%For example, several MU methods require the estimation of the Hessian of the loss function 
%Alternative MU methods draw from the concept of differential privacy \cite{DP_book} and are based on Gaussian noise perturbation of the trained model \cite{DescentToDelete,CertifiedDataRemoval,AdaptiveMachineUnlearning}. The magnitude of the noise perturbation is generally  estimated directly from the clients' data, which is by construction inaccessible to the server in the FL regime.  
%While recent federated unlearning (FU) methods have been proposed to unlearn a client from the global FL model, most of them do not come with theoretical guarantees on the effectiveness of the unlearning \cite{FedEraser,wang2021federated,halimi2022federated,wu2022federated}, or act withing limiting framework assumptions \cite{liu2022fedun, jin2023ntk}.

To address these shortcomings, in this work we introduce \textbf{Sequential Informed Federated Unlearning (SIFU)}, a novel efficient FU approach to remove clients' contributions from the federated model with quantifiable unlearning guarantees. SIFU is compatible with \textsc{FedAvg}-based training and requires minimal additional computations from the server \textit{and none from the clients}. Specifically, at every round of FL optimization, the server quantifies the norm of each client's contribution to the global model.
Upon receiving an unlearning request from a client, the server retrieves the iteration at which the client's contribution exceeds a pre-defined unlearning budget from the FL training history, and initializes the unlearning procedure from the associated intermediate global model. Unlearning guarantees are provided by introducing a novel randomized mechanism to perturb the selected intermediate model with client-specific noise. We develop a theory demonstrating the unlearning capabilities of SIFU in both convex and non-convex FL optimization settings.
We first introduce IFU (Informed FU) to account for only one unlearning request, and then generalize it to SIFU, which can handle an arbitrary number of sequential requests, as is done in \citet{DescentToDelete,AdaptiveMachineUnlearning}.

This manuscript is structured as follows.
In Section \ref{sec:background}, we provide formal definitions for MU, FL, and FU, and introduce the state-of-the art of FU.
In Section \ref{sec:theory}, we introduce sufficient conditions for IFU to unlearn a client from the FL routine (Theorem \ref{theo:noise_DP}). In Section \ref{sec:SIFU}, we extend IFU to the sequential unlearning setting with Sequential IFU (SIFU).  
Finally, in Section \ref{sec:experiments}, we experimentally demonstrate on different tasks and datasets that SIFU leads to more efficient unlearning procedures as compared to basic re-training and state-of-the-art FU approaches.
\section{Background and Related Work}
\label{sec:background}

 Sections \ref{subsec:unlearning_baselines} and \ref{subsec:FedAvg} respectively introduce the basic concepts of MU and FL. The state-of-the-art on FU is discussed in \ref{subsec:federated_unlearning}.

%In Section \ref{subsec:unlearning_baselines}, we introduce elements of the Machine Unlearning state-of-the art, while in Section \ref{subsec:FedAvg}, we introduce FL and \textsc{FedAvg}.
%Finally, we introduce FU  in Section \ref{subsec:federated_unlearning}.

\subsection{Machine Unlearning}
\label{subsec:unlearning_baselines}
Let us consider a dataset $\Dcal$ composed of two disjoint datasets: $\Dcal_f$, the cohort of data samples on which unlearning must be applied after FL training, and $\Dcal_k$, the remaining data samples. Hence, we have $\Dcal = \Dcal_f \sqcup \Dcal_k$. 
We also consider $\Mcal(\Dcal)$, the ML model parameters resulting from training with optimization scheme $\Mcal$ on dataset $\Dcal$. 
We introduce in this section the different unlearning baselines and methods typically used to unlearn $\Dcal_f$ from the trained model $\Mcal(\Dcal)$. 

\textbf{MU through retraining.}
Within this setting, a new training is performed from scratch with only $\Dcal_k$ as training data.
%As the initial model contains no information from $\Dcal_f$, the new trained model $\Mcal(\Dcal_k)$ also contains no information from $\Dcal_f$. 
Retraining from scratch is a typical MU baseline as it provides unlearning by construction, albeit with a generally high computational cost.

\textbf{MU through fine-tuning.} 
Fine-tuning on the remaining data $\Dcal_k$ has been proposed as a practical approach to unlearn the specificities of $\Dcal_f$. This is a common MU baseline \cite{jia2023model}, with however no unlearning guarantees (Appendix \ref{app:sec:initialization}). 

\textbf{MU through model scrubbing.}
Another unlearning approach consists in applying a “scrubbing" transformation $h$ to the model $\Mcal(\Dcal)$ such that the resulting model is as close as possible to the one that would be trained with only $\Dcal_k$, i.e. $h(\Mcal(\Dcal)) \approx \Mcal(\Dcal_k)$ \cite{MakingAIForgetYou}. 
Existing work mostly relies on the quadratic approximation of the loss function to define the scrubbing method $h$ as
% , in which,    for model parameters $\vtheta$ and $\vphi$, the gradient of the loss function of a given data point $\Dcal_x$ satisfies
% \begin{equation}
% \nabla f_{\Dcal_x}(\vphi)
% = \nabla f_{\Dcal_x}(\vtheta)
% + H_{\Dcal_x}(\vtheta)(\vphi - \vtheta)
%     % + o(\vphi - \vtheta)
% ,
% \end{equation}
% where $H_{\Dcal_x}(\vtheta)$ is positive semi-definite.
% The scrubbed model is the new optimum obtained when unlearning data samples in $\Dcal_f$.
% The new optimum can be obtained by setting $\nabla f_{\Dcal_k}(h_{\Dcal_k}(\vtheta)) = 0 $, which gives
\begin{equation}
h_{\Dcal_k}(\vtheta)
= \vtheta - H_{\Dcal_k}^{-1}(\vtheta) \nabla f_{\Dcal_k}(\vtheta)
, \label{eq:scrubbing_method}
\end{equation}
where $H_{\Dcal_x}(\vtheta)$ is the Hessian of the loss function evaluated on the remaining data points $\Dcal_k$.
With equation (\ref{eq:scrubbing_method}), $h$ reduces to performing a Newton step, and has been derived in previous MU works under different theoretical assumptions that can be generalized by considering a quadratic approximation of the loss function \cite{CertifiedDataRemoval,ApproximateDataDeletion,Golatkar_2020_CVPR,golatkar2020forgetting,Golatkar_2021_CVPR,CertifiableMchineUnlearning}. 
The main drawback behind the use of the scrubbing function (\ref{eq:scrubbing_method}) is the estimation of the Hessian, which can be both intractable for large models and prone to information leakage.
Finally, 
the scrubbing function (\ref{eq:scrubbing_method}) is often coupled with Gaussian noise perturbation on the resulting weights, to compensate the quadratic approximation of the loss function, or the approximation of the Hessian \cite{Golatkar_2020_CVPR,golatkar2020forgetting,Golatkar_2021_CVPR}.

\textbf{MU through noise perturbation.}
This unlearning method consists in randomly perturbating the trained model $\Mcal(\Dcal)$ to unlearn specificities from data samples in $\Dcal_f$ \cite{DescentToDelete,AdaptiveMachineUnlearning,mahadevan2021certifiable}. 
The noise is set such that the guarantees of Definition \ref{def:DP_adpated} are satisfied, where $(\epsilon, \delta)$ are parameters quantifying the unlearning guarantees.

\begin{definition}
    \label{def:DP_adpated}  	
    %Let us consider model parameters $\Mcal$ that take a dataset $D$ as input and a function $f_m$ that takes model parameters as an input.
    Let $f_m$ be a randomized mechanism taking model parameters as an input. $(\epsilon, \delta)$-unlearning  trough $f_m$ of a data point $\{x_m, y_m\}$ from a model $\Mcal(D)$  is achieved if, for any subset $\Scal$ of the model parameters space and $D_{-m} \coloneqq D \setminus \{x_m, y_m\}$, we have
	\begin{equation}
	\mathbb{P}(f_m(\Mcal(D)) \in \Scal) 
	\le e^\epsilon \mathbb{P}(f_m(\Mcal(D_{-m})) \in \Scal)  + \delta
	\end{equation}
	\begin{equation}
	\text{and \,}
	\mathbb{P}(f_m(\Mcal(D_{-m})) \in \Scal) 
	\le e^\epsilon \mathbb{P}(f_m(\Mcal(D)) \in \Scal)  + \delta
	.
	\end{equation}
\end{definition}

\cite{CertifiedDataRemoval} shows the relationship between Definition \ref{def:DP_adpated} and the randomized mechanism in Differential Privacy \cite{DP_book,UnderstandingGradientClipping}.

\subsection{Federated Optimization and \textsc{FedAvg}}
\label{subsec:FedAvg}
In FL, we consider a learning setup with $M$ clients, and define $I= \{1, ..., M\}$ as the set of indices of the participating clients. Each client owns a dataset $D_i$ composed of $|D_i| = n_i$ data samples.
We consider a loss $f(\vx_{i, l}, \vy_{i, l}, \vtheta)$ assessed on each data sample $(\vx_{i, l}, \vy_{i, l}) \in D_i$, and define a client's loss function as $f_i(\vtheta) \coloneqq 1/n_i \sum_{l=1}^{n_i} f(\vx_{i, l}, \vy_{i, l}, \vtheta)$. 
% , where $\vx_l$ is the feature vector of a data sample and $\vy_l$ its prediction or label. 
% Each data sample $(\vx_l, \vy_l)$ suffers a loss $f(\vx_l, \vy_l, \vtheta)$, 
We define for the joint dataset $D_I \coloneqq \cup_{i \in I} D_i$ the federated loss function
\begin{equation}\label{eq:empirical_loss}
f_{I}(\vtheta)
% f_{I_x} ( \vtheta)
% \coloneqq \frac{1}{\sum_{i \in I_x}|D_i|} \sum_{i \in I_x} \sum_{(\vx_l, \vy_l) \in D_i} f(\vx_l, \vy_l, \vtheta)
\coloneqq \frac{1}{|D_I|} \sum_{i \in I} |D_i| f_i(\vtheta)
.
\end{equation}
% \begin{equation}\label{eq:empirical_loss}
% f_{I_x}(\vtheta)
% \coloneqq \frac{1}{|D_I|} \sum_{i \in I_x} |D_i| f_i(\vtheta)
% \text{ and }
% f_i(\vtheta) 
% \coloneqq 1/n_i \sum_{l=1}^{n_i} f_i(\vx_{i, l}, \vy_{i, l}, \vtheta)
% .
% \end{equation}
\textsc{FedAvg} \cite{FedAvg} optimizes the loss  (\ref{eq:empirical_loss}) 
% for the joint clients dataset $D_{I_x} \coloneqq \cup_{i \in I_x} D_i$, 
with theoretical guarantees for FL convergence to a stationary point \cite{FedNova,OnTheConvergence}.
%\textsc{FedAvg} \cite{FedAvg} is an FL optimization scheme optimizing the empirical loss function of a distributed dataset like $\Dcal$ or $\Dcal_{-m}$, with convergence guarantees to a stationary point \cite{FedNova, OnTheConvergence}.
%\textsc{FedAvg} is an iterative training strategy based on the aggregation of local model  parameters ${\vtheta}_i^{n}$ to estimate a global model across clients. 
Following Algorithm \ref{alg:FedAvg}, at each step $n$, the server sends the current global model parameters $\vtheta^n$ to the clients. Each client updates the model by minimizing its local cost function $f_i(\vtheta)$ with $K$ SGD steps initialized with $\vtheta^n$. Subsequently each client returns the updated local parameters ${\vtheta}_i^{n+1}$ to the server. The global model parameters $\vtheta^{n+1}$ at iteration step $n+1$ are then estimated by aggregating the clients' contributions, i.e.
\begin{equation}
\label{eq:FedAvg_server_aggregation}
%\vtheta^{n+1}= \frac{1}{|\Dcal|} \sum_{i=1}^M |\Dcal_i| \vtheta_i^{n+1}.
\vtheta^{n+1}= \vtheta^n + \omega(I, \vtheta^n), 
\end{equation}
where $\omega(I, \vtheta^n)=\frac{1}{|D_I|} \sum_{i \in I} |D_i| \left[\vtheta_i^{n+1} - \vtheta^n\right]$. %is the weighted average of the clients contributions at iteration $n+1$.  %, and the output of Algorithm \ref{alg:FedAvg} after $n$ steps on clients $I$ as $\textsc{FedAvg}(I, n)$.

In what follows, we consider \textsc{FedAvg} as reference FL framework, due to the wide adoption of this scheme in the literature, and the lower sensitivity to information leakage as opposed to FedSGD \cite{geng2023improved}.

\begin{algorithm}[t]
	\caption{\textsc{FedAvg}$(I, N)$}\label{alg:FedAvg}
\begin{algorithmic}[1]
	\FOR{$n$ from 0 to $N-1$}
	\STATE The server sends $\vtheta^n$ to every client in $I$.
	\STATE Clients perform $K$ SGDs to compute $\vtheta_i^{n+1}$.
	\STATE The server creates $\vtheta^{n+1}$, equation (\ref{eq:FedAvg_server_aggregation}).
	\ENDFOR
        \RETURN the trained global model $\vtheta^N$ 
\end{algorithmic}
\end{algorithm}

\subsection{Federated Unlearning}
\label{subsec:federated_unlearning}
% Also, MU methods cannot seamlessly be used in the federated setting.

We first note that while MU through retraining and fine-tuning naturally generalize to FU, this not the case for most MU methods, because of typical data access restrictions of FL. 
While a variety of FU methods have been recently proposed, only a few of them offer theoretical proofs of their unlearning capabilities. These include FU methods tailored to clustering tasks \cite{pan2022machine}, linear regressions \cite{li2020online} and class-unlearning \cite{wang2021federated}. Concerning the problem of FU compatible with the \textsc{FedAvg} routine, some recent works develop theories for model scrubbing  \cite{liu2022fedun, jin2023ntk}. Nevertheless, the working assumptions of these methods may be too restrictive as they require the client-wise computation and availability of the model's Hessian, or are based on the existence of independent data available to the server. Other limitations of recent FU methods concern the need for accessing the data to be unlearned after the unlearning request \cite{halimi2022federated}. FedEraser \cite{FedEraser} is a recent method compatible with \textsc{FedAvg} based on adaptive retraining. While this approach has been shown to outperform the retraining baseline, it is not backed by theoretical guarantees.
The ensemble of working hypothesis of these methods is summarized in Table \ref{tab:FU_SOTA}.
As can be observed, there is a lack of a theoretically-proven approaches working with \textsc{FedAvg}, especially in the non-convex setting. With these limitations in mind, in what follows we introduce our contribution.

%Thus, we introduce a novel unlearning method which avoids retraining the final model by identifying the optimal FL iteration where unlearning should be applied. 
%By using a middle-ground between the initialization and the last model, we keep model utility while limiting the amount of noise needed to ensure privacy, thus ensuring faster unlearning. 
%Therefore, retraining is applied to an “early" version of the global model with reduced perturbation, thus minimizing the required amount of SGD steps to achieve convergence.

\begin{table*}
\small
    \caption{Summary of FU methods from state-of-the-art, with related working assumptions.}
    \centering
    \renewcommand{\arraystretch}{0.9} % Adjusting the row spacing
    \begin{tabular}{
        c!{\vrule width 1.5pt}
        c!{\vrule width 0.5pt} % Adding a thin vertical line here
        c|c!{\vrule width 1.5pt} % Moved the 1.5pt line back here
        c|c@{\hspace{0.1cm}}
        c!{\vrule width 1.5pt}
        c|c
    }
        \Xhline{0\arrayrulewidth}
        & & & & \multicolumn{2}{c!{\vrule width 0pt}}{FU does NOT require data from} \\[-\normalbaselineskip]
        \multirow{3}{*}{} & \multirow{0.25}{2cm}{\parbox[c]{2cm}{\centering \textsc{FedAvg}\\ compatible}} & \multicolumn{2}{c!{\vrule width 1.5pt}}{Theoretical guarantees} &  & \\
        \cline{3-4} \cline{5-6}
        & & convex & non-convex & the server & the clients to unlearn \\
        \Xhline{4\arrayrulewidth}
        \citet{liu2022fedun} &  &  \checkmark  &    &  \checkmark  &   \\
        \citet{BFU-SS} &  &    &    &  \checkmark  &  \checkmark \\
        \citet{gong2022forget} &  &    &    &  \checkmark  &  \checkmark \\
        \citet{wu2022federated} & \checkmark &    &   &    &  \checkmark \\
        \citet{halimi2022federated} & \checkmark &    &    &  \checkmark  &   \\
        \citet{FedEraser} & \checkmark &    &    &  \checkmark  &  \checkmark \\
        \citet{jin2023ntk} & \checkmark &  \checkmark  &    &    &   \\
   %     \textcolor{green}{Neel et al. (2021)} & \checkmark & \checkmark   &    & \checkmark   &  \checkmark \\
        \textbf{SIFU} (ours) & \checkmark & \checkmark   &  \checkmark  & \checkmark   &  \checkmark \\
        \Xhline{0\arrayrulewidth}
    \end{tabular}
    \label{tab:FU_SOTA}
\end{table*}

% Considering this state-of-the-art, we focus on FU approaches based on noise perturbation in this work. However, existing MU works on noise perturbation always noise the model resulting from the training procedure. As such noise perturbation incurs a significant convergence slow-down. Hence, the need to improve them in this work by identifying the optimal FL iteration to induce a noise perturbation from, and the associated noise standard deviation sufficient to satisfy Definition \ref{def:DP_adpated}.

% For practicality purposes, an unlearning algorithm must be scalable to a series of forgetting requests \cite{DescentToDelete, AdaptiveMachineUnlearning}, i.e. an unlearning scheme needs to account for a sequence of forgetting requests. 
% Hence, we show that our unlearning algorithm applies to this sequential federated unlearning (SFU) setting. 

%\input{./tex/FU}

%\subsection{Machine Unlearning with Guarantees}
%\label{subsec:DP}
%\input{./tex/DP}

\section{Unlearning a single client with IFU}
\label{sec:theory}

In this section, we develop our theory for the scenario where a model is trained with \textsc{FedAvg} on the set of clients $I$, after which a client $c$ requests unlearning of its data. 
In Section \ref{subsec:bounding_drift}, we define the sensitivity of the global model with respect to a client's contribution, with an associated bound for both convex and non-convex regimes.
% In Section \ref{subsec:tighter_sensitivity}, we provide a tighter model sensitivity for some specific FL applications. 
Using Theorem \ref{theo:diff_bound}, we introduce the perturbation procedure in Section \ref{subsec:unlearning_guarantees} to unlearn a client $c$ from the model trained with \textsc{FedAvg}. %(Theorem \ref{theo:noise_DP}). 
Finally, using Theorem \ref{theo:noise_DP}, we introduce Informed Federated Unlearning (IFU) (Algorithm \ref{alg:unlearning_ours}).

\subsection{Bounding the Model Sensitivity}
\label{subsec:bounding_drift}
% \subsection{Theorem \ref{theo:diff_bound}, Bounding the Model Sensitivity}
%As defined in equation (\ref{eq:FedAvg_server_aggregation}), 
%$\vtheta_i^{n+1}$ is the local update of client $i$ sent to the server after performing $K$ SGD steps on its dataset $D_i$ after initialization with global model $\vtheta^n$. 
In what follows, we define the joint dataset for a subset of client $I_x \subset I$ as $D_{I_x} \coloneqq \cup_{i \in I_x} D_i$. Additionally, for any given client $c$, we define $I_{-c} \coloneqq I \setminus \{c\}$.
% Given the contribution $\left(\vtheta_i^{n+1} - \vtheta^n\right)$ of a client $i$ to the federated aggregation of Equation (\ref{eq:FedAvg_server_aggregation}), we define the overall FL increment after aggregations across the set of clients $I$ as
% \begin{equation}
% \delta(I, \vtheta^n)
% \coloneqq \frac{1}{|D_I|}\sum_{i \in I} |D_i| \left[\vtheta_i^{n+1} - \vtheta^n\right] 
% \label{eq:def_server_contrib}
% .
% \end{equation} 

We introduce the \textit{model sensitivity} with respect to client $c$ after $n$ aggregation rounds of \textsc{FedAvg} as  
\begin{align}
\label{eq:model_sensitivity}
        \alpha(n, c)
        &\coloneqq \norm{\textsc{FedAvg}(I, n) - \textsc{FedAvg}(I_{-c}, n)}_2,
\end{align} 
where $\textsc{FedAvg}(I, n)$ it the global model obtained by applying Algorithm \ref{alg:FedAvg} for $n$ iterations over the data of clients in $I$. While the model sensitivity is an ideal measure of the impact of a given client on the federated optimization result at step $n$, the computation of this quantity for each client is not feasible in a typical FL routine. We therefore introduce a proxy for this quantity, to keep track at every FL round of each client's contribution to the aggregation  (\ref{eq:FedAvg_server_aggregation}):
\begin{equation}\label{eq:delta}
    \Delta_c(I, \vtheta) \coloneqq \norm{\omega(I,\vtheta)-\omega(I_{-c}, \vtheta)}
\end{equation}

In Theorem \ref{theo:diff_bound}, we establish a bound for the \textit{model sensitivity}, relating this quantity to the history of updates provided by the clients across FL rounds.
\begin{theorem}
     \label{theo:diff_bound}
    For smooth client's local loss functions (i.e. with Lipschitz-continuous gradients), 
    % Under the assumption of Lipschitz smoothness for the gradients of the clients' loss functions, %Under smoothness assumption (i.e. with Lipschitz gradient), 
    we have
	% Under Assumption \ref{ass:linear_approx}, the sensitivity of \textsc{FedAvg}, Algorithm \ref{alg:FedAvg},  when removing client $c$ after $n$ server aggregations is defined as
    \begin{align}
	\alpha(n , c)
	&\le \Psi(n , c),
    \end{align} 
    with the bounded sensitivity $\Psi$ defined as:
      \begin{align}
    \Psi(n, c)
    = \sum\limits_{s=0}^{n-1} B(f_I, \eta) ^{\gamma_{s,n}}\cdot \Delta_c (I, \vtheta^s)
    \label{eq:def_Psi},
    \end{align}
    where $\eta$ is the learning rate, $\gamma_{s,n} = (n-s-1)K$, and %$B(f_I, \eta)$ accounts for the regularity of the clients' loss function. We respectively have 
    $B(f_I, \eta )<1$, $B(f_I, \eta )=1$ or $B(f_I, \eta )>1$ if the clients' loss functions are smooth and, respectively, strongly convex, convex, or non-convex. The exact formula for $B(f_I, \eta)$ is given in Appendix \ref{app:sec:proof_diff}, equations (\ref{eq:B1}) to (\ref{eq:B3}).
\end{theorem}

\begin{proof}
We prove Theorem \ref{theo:diff_bound} in Appendix \ref{app:sec:proof_diff}. %We also derive in Appendix \ref{app:sec:proof_diff} similar results when considering weaker assumptions for the gradient approximation than the ones of Assumption \ref{ass:linear_approx}.
\end{proof}

\subsection{From model sensitivity to certified unlearning}
% \subsection{Satisfying Definition \ref{def:DP_adpated}}
\label{subsec:unlearning_guarantees}
In this section, we introduce a randomized mechanism to provide guarantees for the unlearning of a given client $c$, where the magnitude of the perturbation process \cite{DP_book} is defined based on the sensitivity of Theorem \ref{theo:diff_bound}. In practice, we define a Gaussian noise mechanism to perturb each parameter of the global model $\vtheta^n$ such that we achieve $(\epsilon, \delta)$-unlearning of client $c$, according to Definition \ref{def:DP_adpated}. 
We give in Theorem \ref{theo:noise_DP} sufficient conditions for the noise perturbation to satisfy Definition \ref{def:DP_adpated}. 

% With Theorem \ref{theo:diff_bound}, we bound the trained model sensitivity with \textsc{FedAvg} when learning with clients in $I$ versus in $I_{-c}$.   
% We remove the data specificities of client $c$ with the Gaussian noise mechanism proposed in \cite{DP_book}. We add to each parameter of global model $\vtheta^n$ a Gaussian noise with 0 mean and standard deviation $\sigma(n, c)$ such that the resulting model $(\epsilon, \delta)$-unlearns client $c$, Definition \ref{def:DP_adpated}. 
% We give in Theorem \ref{theo:noise_DP} sufficient conditions for the noise perturbation to satisfy Definition \ref{def:DP_adpated}. 

\begin{theorem}
\label{theo:noise_DP} 
Under smoothness assumptions, applying to the global model $\vtheta^n$ a Gaussian noise $N(0,\sigma(n , c)^2 \mI_\vtheta)$, with $\mI_\vtheta$ the identity matrix and
% Under the smoothness assumption, and by defining 
    % 
    % When applying to every parameter in $\vtheta^n$ the Gaussian perturbation 
    \begin{equation}
        \sigma(n, c) 
        := \left[2 \left( \ln(1.25) - \ln(\delta) \right) \right]^{1/2} \epsilon^{-1} \Psi(n , c),
    \end{equation}
    % the resulting perturbed model $(\epsilon, \delta)$-unlearns, Definition \ref{def:DP_adpated}, client $c$.
% the noise perturbation $\sigma(n , c) \mI_\vtheta$ applied to the global model $\vtheta^n$, where $\mI_\vtheta$ is the identity matrix, achieves $(\epsilon, \delta)$-unlearning of client $c$ according to Definition \ref{def:DP_adpated}.
achieves $(\epsilon, \delta)$-unlearning of client $c$ according to Definition \ref{def:DP_adpated}.
\end{theorem}

\begin{proof}
	While a formal proof is given in Appendix \ref{app:sec:proof_DP}, this statement follows directly from Theorem \ref{theo:diff_bound} coupled with Theorem A.1 of \cite{DP_book}.
\end{proof}

We note that, according to Theorem \ref{theo:noise_DP}, $(\epsilon, \delta)$-unlearning a client from a given global model requires to prescribe a client-specific standard deviation for the noise, proportional to the bounded sensitivity. This is not an issue, as the bounded sensitivity is a scalar quantity that can be easily computed and stored from the clients' contribution. %We also note that, similarly to differentialy private SGD, we can adopt gradient clipping to define a common sensitivity bound across clients \cite{mcmahan2017learning, mcmahan2018general, abadi2016deep}.

%Two conclusions can be drawn from Theorem \ref{theo:noise_DP}. 
%(1) $(\epsilon, \delta)$-Unlearning a client from a given global model requires a standard deviation for the noise that is client-specific and proportional to its bounded sensitivity. 
%(2) When $(\epsilon, \delta)$-unlearning a client with a bounded standard deviation for the noise, i.e. $\sigma(n, c) \le \sigma$, the most recent global model to perturb is client-specific and based on the bounded sensitivity. In Section \ref{subsec:unlearning_ours}, we introduce IFU based on this second point. 
% (2) $(\epsilon, \delta)$-unlearnting a client with a bounded standard deviation for the noise, i.e. $\sigma(n, c) \le \sigma$, gives potentially a different global model to perturb for each client and is solely based on its bounded sensitivity. In Section \ref{subsec:unlearning_ours}, we introduce IFU based on this second point. 
In what follows, the unlearning procedure will be defined with respect to the sensitivity threshold $\Psi^*$ related to the unlearning budget $(\epsilon, \delta)$ and standard deviation $\sigma$:
% \begin{equation}
% 	\Psi^*
% 	\coloneqq \sigma \epsilon/ c
%  .
%  \label{eq:psi_star}
% \end{equation}
\begin{equation}
\Psi^*
\coloneqq \left[2 \left(\ln(1.25) - \ln(\delta)\right)\right]^{-1/2} \epsilon \sigma 
 \label{eq:psi_star}
 .
\end{equation}

\subsection{Informed Federated Unlearning (IFU)}
\label{subsec:unlearning_ours}
\begin{algorithm}[]
	\caption{Informed Federated Unlearning (IFU)}\label{alg:unlearning_ours}
	
	\textsc{Learning with \textsc{FedAvg}}
	\begin{algorithmic}		
		\STATE \textsc{FedAvg}($I, N$) initialized on initial model $\vtheta^0$.
		\FOR{$n$ from 0 to $N-1$, and $i$ from 1 to $c$}
		\STATE Compute $\Psi(n, i)$, equation (\ref{eq:def_Psi}).
		\ENDFOR		
	\end{algorithmic}
	\textsc{Unlearning}
	\begin{algorithmic}[1]
		\REQUIRE $c$, $\epsilon$, $\delta$, $\sigma$, amount of retraining steps $\tilde{N}$.
      \STATE Get $\Psi^*$ with equation (\ref{eq:psi_star}).
		\STATE Get $T = \argmax_n \left(\Psi(n, c) \le \Psi^* \right)$.
		\STATE The new global model is $\tilde{\vtheta} = \vtheta^T + N(0, \sigma^2 \mI_\vtheta)$.
		\STATE Run \textsc{FedAvg}($I_{-c}$, $\tilde{N}$) initialized on $\tilde{\vtheta}$.
	\end{algorithmic}
\end{algorithm}

Using the bounded sensitivity (\ref{eq:def_Psi}) and Theorem \ref{theo:noise_DP}, we introduce Informed Federated Unlearning (IFU) to unlearn the contribution of client $c \in I$ from a FL training procedure based on \textsc{FedAvg}. Algorithm \ref{alg:unlearning_ours} provides the implementation of IFU on top of \textsc{FedAvg}. %At every optimization round, clients perform $K$ SGD steps on the model they receive, and the server aggregates their local work. 
We note that during the FL training, IFU requires the server to compute the bounded sensitivity metric $\Psi(n, i)$ from each client's contribution $\vtheta_i^{n+1}$ and current global model $\vtheta^n$. These quantities are tracked throughout FL iterations and are used to identify the optimal unlearning strategy after request from a client $c$. 
We note that since the IFU procedure of Algorithm \ref{alg:unlearning_ours} relies on \textsc{FedAvg}, the convergence guarantees are the same as in \cite{FedNova, OnTheConvergence}.

To unlearn client $c$, the server identifies the unlearning index $T$ associated to the history of bounded sensitivity metrics, i.e. the most recent global model index such that a perturbation of size $\sigma$ satisfies Theorem \ref{theo:noise_DP}:
\begin{equation}
	T
	\coloneqq \max \{n~:~\Psi(n,c) \leq \Psi^*\}
	\,.\label{eq:T_IFU}
\end{equation}
The new global model is obtained after perturbation $\tilde{\vtheta} \coloneqq \vtheta^T + \nu$, where $\nu \sim N(0, \sigma^2 \mI_\vtheta)$. Our unlearning criterion \ref{def:DP_adpated} is therefore satisfied for $\tilde{\vtheta}$ (Theorem \ref{theo:noise_DP}), and the server can perform $\tilde{N}$ new optimization rounds with \textsc{FedAvg} initialized with $\tilde{\vtheta}$. Thanks to the contribution of the remaining clients in $\tilde{\vtheta}$, the retraining with IFU is generally faster than retraining with a random initial model. 
% IFU requires to keep track at each round of the global model and of the sensitivities computed for each client.

Since $\Psi(n, i)$ increases with $n$, the server can stop computing the bounded sensitivity (\ref{eq:def_Psi}) for client $i$ whenever the condition $\Psi(n_i, i ) > \Psi^*$ is verified after $n_i$ optimization rounds. 
At this point, the model $\vtheta^{n_i-1}$ will be selected for the unlearning request of client $i$, as the models at subsequent iterations do not comply with the desired unlearning budget $\Psi^*$. Thus, Eq \ref{eq:T_IFU} does not need the entire history of the global models to be performed. Only one version of the global model must be kept for each client potentially wishing to be unlearned in latter stages.
We also note that computing the bounded sensitivity (\ref{eq:def_Psi}) only requires to compute the norm of sums of vectors already computed, which can be done while the clients perform their local updates. Hence, there is no added time required to compute the bounded sensitivity.

If the clients wish to employ gradient masking techniques to avoid revealing their full updates, each client can compute $\Delta_c (I, \vtheta)$ for themselves. Indeed, Eq. \ref{eq:delta} can be re-written as:
\begin{equation}
    \Delta_c(I, \vtheta^n) = \frac{|D_c|}{|D_I| - |D_c|} \norm{\vtheta^{n+1}-\vtheta^{n+1}_c}
\end{equation}

% \subsection{Sequential FU through an example}
% \label{subsec:SIFU_example}
% \input{./tex/SIFU_example}

% \subsection{Generalization with Sequential IFU }
% \label{subsec:generalization}
\section{Sequential FU with SIFU}
\label{sec:SIFU}

\begin{algorithm}[t]
\caption{Sequential IFU (SIFU)}\label{alg:SIFU}

\textsc{Learning with \textsc{FedAvg}}
\begin{algorithmic}[1]
    \STATE \textsc{FedAvg}($I, N$) initialized on initial model $\vtheta_0^0$. 
    \STATE Compute $\Psi_0(n, i)$, equation (\ref{eq:def_Psi_extended}).
\end{algorithmic}
	
\textsc{Unlearning the series of requests} \mbox{$\{W_u\}$}
\begin{algorithmic}[1]
    \REQUIRE $\{W_u\}_{u=1}^R$, $\epsilon$, $\delta$, $\sigma$, and $\{N_u\}_{u=1}^U$
    \STATE Get $\Psi^*$ with equation (\ref{eq:psi_star}).
    \FOR{$u$ from 1 to $U$}
        \STATE $I_u = I_{u-1} \setminus W_u$.
        \STATE Compute ($\zeta_u$, $T_u$) with $H(u)$, eq. (\ref{eq:theta_T_zeta}) and (\ref{n_u+1}).\STATE The new global model is $\vtheta_u^0 = \vtheta_{\zeta_u}^{T_u} + N(0, \sigma^2 \mI_\vtheta)$.\STATE Perform \textsc{FedAvg}($I_u, N_u$) initialized on $\vtheta_u^0$.
        \STATE Update $H(u+1)$ with $\zeta_u$, $T_u$, and $H(u)$, eq. (\ref{eq:H_recc}).\STATE Compute $\Psi_u(n, I_u)$, eq. (\ref{eq:def_Psi_extended_set}).
    \ENDFOR
\end{algorithmic}

\end{algorithm}

	\def\spacesquare{1.15} %default value for space between node values in the feature space.
	\def\h{ 0.9}
	\def\squareborder{0.5}

	\begin{figure*}[ht]
		\begin{center}
			
		\begin{tikzpicture}[
%			scale = 1, 
%			line cap=round, 
%			line join=round, , 
%			x=1cm, 
%			y=1cm
			bluenode/.style={draw, fill=blue, line width=\squareborder},
			greennode/.style={draw, fill=green, line width=\squareborder},
			rednode/.style={draw, fill=red, line width=\squareborder},
			blacknode/.style={draw, fill=black, line width=\squareborder},
			orangenode/.style={draw, fill=orange, line width=\squareborder},
			arrowSGD/.style={line width=0.5pt, ->, 
				>=triangle 45},
			arrowPerturb/.style={line width=0.5pt, ->, 
				>= open triangle 45},
		]

		% Legend 
        \matrix [draw] at (-1.5, 3.02) {
            \node [bluenode] {}; & 
            % \node{Training, $O(0) = \{(0, 0)\}$}; 
            \node{Training}; 
            \\
            \node [rednode] {}; & 
            % \node{Unlearning $W_1$, $O(1) = \{(0, 0), (0, T_1)\}$, $\zeta_1=0$}; 
            \node{Unlearning $W_1$}; 
            \\
            \node [greennode] {}; & 
            % \node{Unlearning $W_2$, $O(2) = \{(0, 0), (0, T_1), (1, T_2)\}$, $\zeta_2=0$}; 
            \node{Unlearning $W_2$}; 
            \\
            \node [orangenode] {}; & 
            % \node{Unlearning $W_3$, $O(3) = \{(0, 0), (0, T_1), (1, T_3)\}$, $\zeta_3=0$}; 
            \node{Unlearning $W_3$}; 
            \\

            \draw [arrowSGD] (0, 0) -- (0.6, 0); & \node{Server aggregation};\\
            \draw [arrowPerturb] (0, 0) -- (0.6, 0); & \node{Noise perturbation};\\
		};

        \matrix [draw] at (5.1, 4.5) {
            \node{ $H(0) = (\vtheta_0^0, \ldots,  \vtheta_0^{N_0})$}; \\
            \node{ $H(1) = (\vtheta_0^0, \ldots, \vtheta_0^{T_1}, \vtheta_1^0, \ldots, \vtheta_1^{N_1})$}; & \node{$\zeta_1=0$}; \\
            \node{$H(2) =(\vtheta_0^0, \ldots, \vtheta_0^{T_1}, \vtheta_1^0, \ldots, \vtheta_1^{T_2}, \vtheta_2^{0}, \ldots \vtheta_2^{N_2})$}; & \node{ $\zeta_2=1$}; \\
            \node{ $H(3) =(\vtheta_0^0, \ldots, \vtheta_0^{T_1}, \vtheta_1^0, \ldots, \vtheta_1^{T_3}, \vtheta_3^{0}, \ldots \vtheta_3^{N_3})$}; & \node{ $\zeta_3=1$}; \\
		};

%		

		%NODES TRAINING
		\node[bluenode, label=left:{$\vtheta_0^0$}] (A1) at (0, 0) {};
		\node[bluenode, label=below:{$\vtheta_0^1$}] (A2) at (\spacesquare, 0.3) {};
		\node[bluenode] (A3) at (2 *\spacesquare, 0.5) {};
		\node[bluenode] (A4) at (3*\spacesquare, 0.6) {};
		\node[bluenode, label=below:{$\vtheta_0^{T_1}$}] (A5) at (4*\spacesquare, 0.6) {};
		\node[bluenode] (A6) at (5*\spacesquare -.1, 0.5) {};
		\node[bluenode] (A7) at (6*\spacesquare - 0.2, 0.35) {};
		\node[bluenode] (A8) at (7*\spacesquare -0.3, 0.2) {};
		\node[bluenode, label=below:{}] (A9) at (8*\spacesquare - 0.5, 0.1) {};
		\node[bluenode, label=below:{}] (A10) at (9*\spacesquare - 0.7, 0) {};
		\node[bluenode, label=right:{$\vtheta_0^{N_0}$}] (A11) at (10*\spacesquare - 1.2, -0.05) {};
		
		%NODES r = 1
		\node[rednode, label=left:{$\vtheta_1^0$}] (B1) at (4*\spacesquare + 0.1, 2 *\h) {};
		\node[rednode] (B2) at (5*\spacesquare -0.1, 2 * \h + 0.1) {};
		\node[rednode, label=above:{$\vtheta_1^{T_3}$}] (B3) at (6*\spacesquare -0.3, 2 * \h +0.15) {};
		\node[rednode, label=below:{$\vtheta_1^{T_2}$}] (B4) at (7*\spacesquare -0.5, 2 * \h +0.1) {};
		\node[rednode] (B5) at (8*\spacesquare -0.8, 2 * \h -0.1) {};
		\node[rednode] (B6) at (9*\spacesquare -1.1, 2 * \h - 0.5) {};
		\node[rednode, label=right:{$\vtheta_1^{N_1}$}] (B7) at (10*\spacesquare - 1.5, 2 * \h - 1.) {};
		
		%NODES r =2
		\node[greennode, label=above:{$\vtheta_2^0$}] (C1) at ( 7 * \spacesquare, 2.5 * \h) {};
		\node[greennode] (C2) at ( 8 * \spacesquare ,  2.5 * \h + 0.05) {};
		\node[greennode] (C3) at ( 9 * \spacesquare - 0.4,  2.5 * \h +0.1) {};
		\node[greennode, label=right:{$\vtheta_2^{N_2}$}] (C4) at ( 9 * \spacesquare + 0.4,  2.5 * \h + 0.3) {};
		% \node[greennode, label=right:{$\vtheta_2^{N_2}$}] (C5) at ( 10 * \spacesquare,  2.5 * \h + 0.5) {};
		
		%NODES r = 3
		\node[orangenode, label=left:{$\vtheta_3^0$}] (D1) at (5 * \spacesquare + 0.5, \h + 0.2) {};
		\node[orangenode] (D2) at (6 * \spacesquare + 0.2, \h -0.1) {};
		\node[orangenode] (D3) at (7 * \spacesquare -0.2, \h -0.2) {};
		\node[orangenode, label=right:{$\vtheta_3^{N_3}$}] (D4) at (7 * \spacesquare + 0.6, \h -0.1) {};
		% \node[orangenode, label=right:{$\vtheta_3^{N_3}$}] (D5) at (8 * \spacesquare, \h ) {};

		%ARROWS TRAINING
		\draw [arrowSGD] (A1)-- (A2);
		\draw [arrowSGD] (A2)-- (A3);
		\draw [arrowSGD] (A3)-- (A4);
		\draw [arrowSGD] (A4)-- (A5);
		\draw [arrowSGD] (A5)-- (A6);
		\draw [arrowSGD] (A6)-- (A7);
		\draw [arrowSGD] (A7)-- (A8);
		\draw [arrowSGD] (A8)-- (A9);
		\draw [arrowSGD] (A9)-- (A10);
		\draw [arrowSGD] (A10)-- (A11);

		% Arrows r=1
		\draw [arrowSGD] (B1)-- (B2);
		\draw [arrowSGD] (B2)-- (B3);
		\draw [arrowSGD] (B3)-- (B4);
		\draw [arrowSGD] (B4)-- (B5);
		\draw [arrowSGD] (B5)-- (B6);
		\draw [arrowSGD] (B6)-- (B7);
		
		% Arrows r=2
		\draw [arrowSGD] (C1)-- (C2);
		\draw [arrowSGD] (C2)-- (C3);
		\draw [arrowSGD] (C3)-- (C4);
		% \draw [arrowSGD] (C4)-- (C5);
		
		% Arrows r=1
		\draw [arrowSGD] (D1)-- (D2);
		\draw [arrowSGD] (D2)-- (D3);
		\draw [arrowSGD] (D3)-- (D4);
		% \draw [arrowSGD] (D4)-- (D5);
		
		% Unlearning 
		\draw [arrowPerturb] (A5)-- (B1);
		\draw [arrowPerturb] (B4)-- (C1);
		\draw [arrowPerturb] (B3)-- (D1);

		\end{tikzpicture}
		
		\end{center}
		
\caption{
Illustration of SIFU (Algorithm \ref{alg:SIFU}) when the server receives $U=3$ unlearning requests, through the evolution of the global model parameters $\vtheta_u^n$ after server aggregation and noise perturbation. 
After standard federated training via \textsc{FedAvg}$(I, N_0)$ the training history is $H(0) = (\vtheta_0^0, \ldots,  \vtheta_0^{N_0})$. At request $u=1$, the unlearning index is $T_1$, and the training history becomes $H(1) = (\vtheta_0^0, \ldots, \vtheta_0^{T_1}, \vtheta_1^0, \ldots, \vtheta_1^{N_1})$ with $\zeta_1 = 0$. At request $u=2$, the unlearning index is $T_2$ and the training history becomes $H(2) =(\vtheta_0^0, \ldots, \vtheta_0^{T_1}, \vtheta_1^0, \ldots, \vtheta_1^{T_2}, \vtheta_2^{0}, \ldots \vtheta_2^{N_2})$ with $\zeta_2 = 1$. Finally, at request $u=3$, the unlearning index is found at $T_3<T_2$ in the branch of request $u=1$. The updated training history is now $H(3) =(\vtheta_0^0, \ldots, \vtheta_0^{T_1}, \vtheta_1^0, \ldots, \vtheta_1^{T_3}, \vtheta_3^{0}, \ldots \vtheta_3^{N_3})$ with $\zeta_3 = 1$.
}
\label{fig:example_with_R3}
\end{figure*}
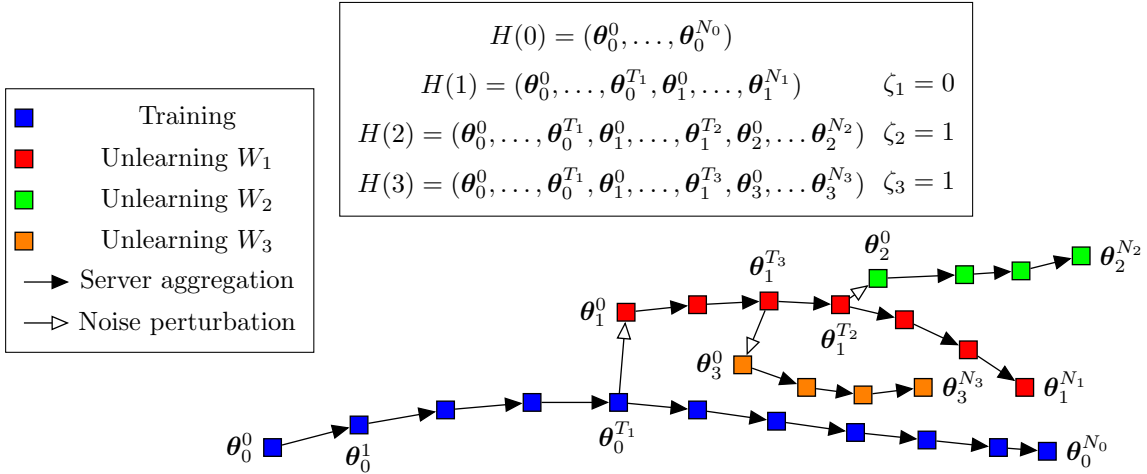

In this section, we extend IFU to the sequential unlearning setting with Sequential IFU (SIFU). With Algorithm \ref{alg:SIFU}, SIFU is designed to satisfy a series of $U$ unlearning requests expressed by set of indices $W_u$, corresponding to clients to unlearn at request $1\leq u\leq U$. 
SIFU generalizes IFU for which $U=1$ and $W_1 = \{c\}$. We provide an illustration of SIFU in Figure \ref{fig:example_with_R3}.

The notations introduced thus far need to be generalized to account for series of unlearning requests $W_1, W_2, \ldots, W_U$.  Global models are now referenced by their coordinates $(u, n)$, i.e. $\vtheta_u^n$ represents the model after $u$ unlearning requests followed by a retraining made of $n$ aggregation steps. Hence, $\vtheta_u^0$ is the initialization of the model when starting to unlearn the clients in $W_u$. 
Additionally, we define $N_u$ as the number of server aggregations on the remaining clients required to reach the desired performance threshold (i.e. perform successful retraining).
Therefore, by construction,  $\vtheta_u^{N_u}$ is the model obtained after using SIFU to process the sequence of unlearning requests $\{W_s\}_{s=1}^u$. Finally, we define $I_u$ as the set of remaining clients after unlearning request $u$, i.e. $I_u \coloneqq I\setminus \cup_{s=1}^u W_s = I_{u-1} \setminus W_u$, with $I_0 = I$.

% In Section \ref{sec:theory}, we considered a set of $I$ clients from which we unlearn client $c$. We need to generalize our notations to the sequential unlearning setting and account for their dependency on the unlearning request index $r$ in addition of the amount of server aggregations $n$.
In case of multiple unlearning requests, the bounded sensitivity (\ref{eq:def_Psi}) for client $i$ must be updated at each unlearning index $u$ to account for the new history of global models resulting from retraining. 
With SIFU, the selection of the unlearning index $T$ for a request $u$ depends of the past history of unlearning requests. To  track of the evolution of the unlearning procedure, we introduce the \emph{model history} $H(u)$, which keeps track of each iteration of the global model across requests. Please note that this object is here introduced solely for illustration purposes, and is not actually stored when running SIFU. With reference to Figure \ref{fig:example_with_R3}, we start with the original sequence of global models obtained at each FL round, i.e. $H(0) = (\vtheta_0^0, \ldots, \vtheta_0^{N_0})$. Similarly to IFU, the first unlearning request requires to identify the unlearning index $T_1$ for which the corresponding global model $\vtheta_0^{T_1}$ must be perturbed to obtain $\vtheta_1^0$ and retrained until convergence, i.e. up to $\vtheta_1^{N_1}$. In this case, the bounded sensitivity is computed according to equation (\ref{eq:def_Psi}).

After unlearning (i.e. Gaussian perturbation followed by re-training), the current training history is now $H(1)= (\vtheta_0^0, \ldots, \vtheta_0^{T_1}, \vtheta_1^0, \ldots, \vtheta_1^{N_1})$. 
More generally, we define $H(u)$ as the training history after $u$ unlearning requests and $N_u$ FL iterations steps from $\vtheta^u_0$. We define the increment history of a given client $c$ as the sequence obtained by computing $\Delta_c(I_k, \vtheta_k^s)$ on every element of $H$, according to their order of appearance in the training history:
\begin{equation}
    \Lambda_c^u = (\Delta_c(I_k, \vtheta_k^s) \text{  for  } \vtheta_k^s \in H(u))
    ,
\end{equation}
% We will denote the elements in $\Lambda_c$ using the following notation : $(\Delta_{c,k}) \coloneqq \Lambda_c(u,n)$, where $0 \le k \le L(u,n)$.
The bounded sensitivity for client $i$ should be updated to account for this new history of global models. We therefore generalize equation (\ref{eq:def_Psi}) to account for the entire training history:
% \begin{equation}
% \psi_r(n,i) = \sum\limits_{s=0}^{n-1} B(f_{I}, \eta)^{\gamma_{s}}\cdot\norm{\Delta(I_r,\vtheta_r^s)-\Delta(I_r\setminus\{i\}, \vtheta_r^s)},
% \label{eq:def_Psi_extended}
% \end{equation}
\begin{equation}
\Psi_u(n, c)
\coloneqq \sum_{s=1}^{n} B(f_{I}, \eta)^{\gamma_{s, n}} \cdot \Lambda_{c}^u[s]
,
\label{eq:def_Psi_extended}
\end{equation}
% \begin{equation}
% \Psi_r(n, i)
% \coloneqq \sum_{s = 0}^{L_H-1} B(f_{I}, \eta)^{\gamma_{n}} \cdot \Delta_c(I_{???}, \vtheta_k)
% .
% \end{equation}
where $\Lambda_{c}^u[s]$ is the $s$-th element of the sequence $\Lambda_{c}^u$.
We extend this quantity to a set of clients $S$ as
\begin{equation}
	\Psi_u(n, S)  
	\coloneqq \max_{c \in S}\Psi_u(n, c)
	\label{eq:def_Psi_extended_set}
	.
\end{equation}

 For a given unlearning request $u+1$, we evaluate $\Psi_u(n, W_{u+1})$ along the clients' history and we identify the optimal parameters to initialize unlearning from:
 \begin{equation} \label{eq:theta_T_zeta}
     \vtheta_{\zeta_{u+1}}^{T_{u+1}} = H(u)[n_{u+1}],
 \end{equation} where
\begin{equation}\label{n_u+1}
n_{u+1} = \max \{n~:~\Psi_u(n,W_{u+1}) \leq \Psi^*\}\,.
\end{equation}

We proceed by perturbing the parameters $\vtheta_{\zeta_{u+1}}^{T_{u+1}}$ with Gaussian noise defined in Theorem \ref{theo:noise_DP} to obtain $\vtheta_{u+1}^0$. A new FL routine is then operated with the remaining clients to obtain parameters $\vtheta_{u+1}^{N_{u+1}}$, and to update the training history as 
\begin{equation}\label{eq:H_recc}
H(u+1)= (\vtheta_0^0, \ldots, \vtheta_{\zeta_{u+1}}^{T_{u+1}}, \vtheta_{u+1}^0, \ldots, \vtheta_{u+1}^{N_{u+1}}).
\end{equation}

As for IFU, performing SIFU requires the server to store the sensitivity associated to each client at each round, and one global model checkpoint per client. The computational cost for this operation is negligible. 
Theorem \ref{theo:zeta} shows that with SIFU we achieve unlearning guarantees for every client in a sequential unlearning request $F_u$. 
%Theorem \ref{theo:zeta} shows that for a model trained with SIFU after a given training request $r$, $(\epsilon, \delta)$-unlearning is guaranteed for every client belonging to $W_s$, $s\leq r$.
\begin{theorem}
\label{theo:zeta}
The model $\vtheta_u^{N_u}$ obtained with SIFU satisfies $(\epsilon, \delta)$-unlearning for every client in current and previous unlearning requests, i.e. in $F_u = \cup_{s=1}^u W_s$.
\end{theorem}
\begin{proof}
    See Appendix \ref{app:sec:SIFU_convergence}.
\end{proof}

\section{Experiments}
\label{sec:experiments}

\begin{figure*}[ht]
\begin{centering}
    \includegraphics[width=\linewidth]{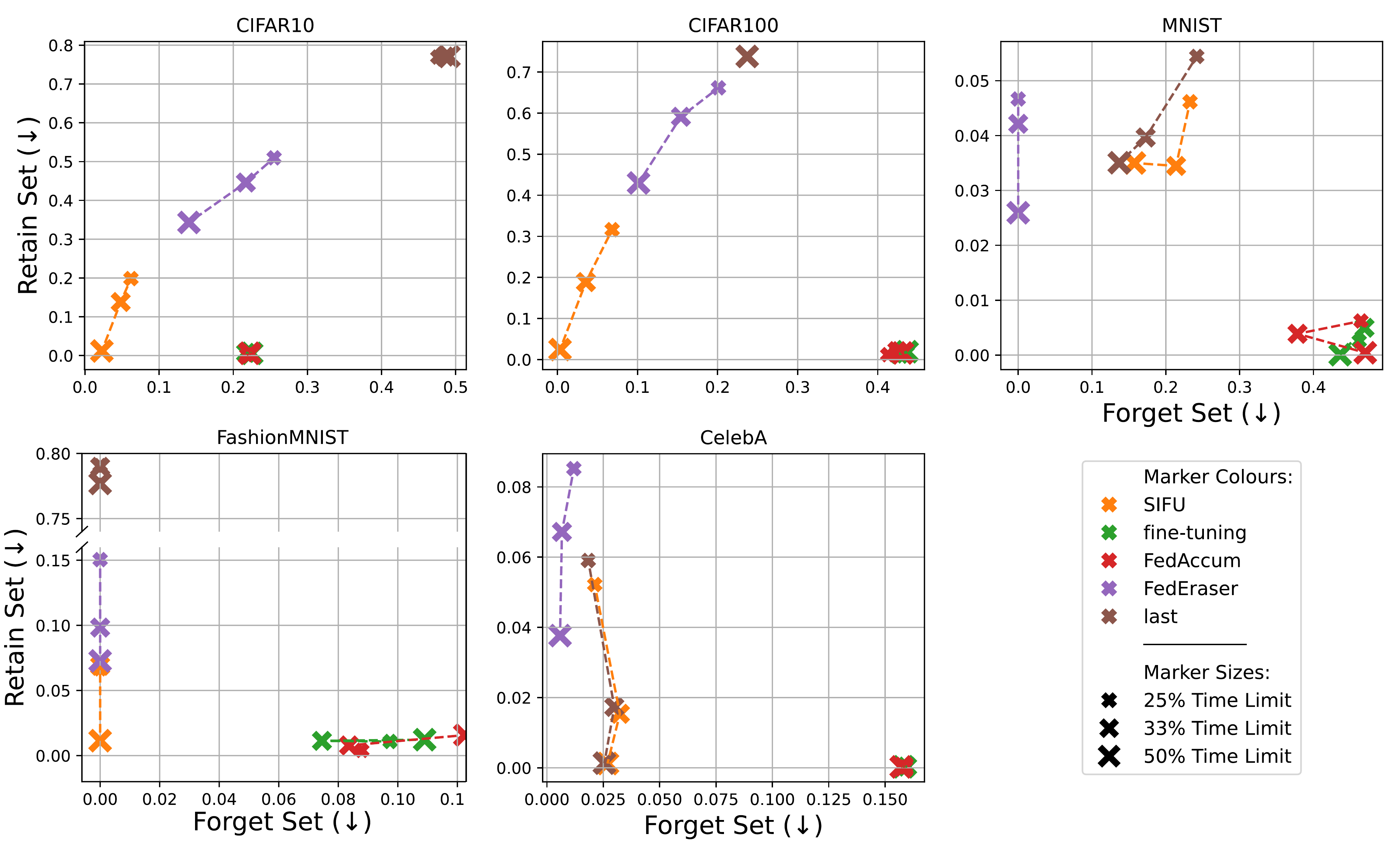}
\end{centering}
\caption{ Difference in accuracy (absolute value) between \textsc{Scratch} and the considered unlearning methods, on both retain and forget sets (lower is better).
}
	\label{fig:NEW_SIFU}
\end{figure*}

In this section, we experimentally demonstrate the effectiveness of SIFU through a series of experiments introduced in Section \ref{subsec:experimental_setup}. In Section \ref{subsec:experimental_results}, we illustrate and discuss our experimental results.
Results and related code are publicly available at \href{https://github.com/Accenture/Labs-Federated-Learning/tree/SIFU}{this GitHub URL}.

\subsection{Experimental Setup}
\label{subsec:experimental_setup}

\textbf{Datasets.} We report experiments on adapted versions of CIFAR-10 \cite{CIFAR-10}, CIFAR-100 \cite{CIFAR-10}, MNIST \cite{MNIST}, FashionMNIST \cite{FashionMNIST}, and CelebA \cite{CelebA}. For each dataset, we consider $M=100$ clients, with 100 data points each. For MNIST and FashionMNIST, each client has data samples from only one class, so that each class is represented in 10 clients only. For CIFAR10 and CIFAR100, each client has data samples with ratio sampled from a Dirichlet distribution with parameter 0.1 \cite{FL_and_CIFAR_dir}.
Finally, in CelebA, clients own data samples representing the same celebrity as done in LEAF \cite{Leaf}. With these five datasets, we consider different levels of heterogeneity based on labels and features distribution.

\textbf{Models.}
For MNIST, we train a logistic regression model to consider a convex classification problem. For the other four datasets, we train a neural network with convolutional layers followed by fully connected ones. Further details are provided in Appendix \ref{app:sec:experiments}.

\textbf{Unlearning schemes.}
% We compare SIFU to \textsc{Scratch}, where retraining of a new initial model is performed on the remaining clients. \textsc{Scratch} is the best unlearning method but requires to retrain from a random model which also makes \textsc{Scratch} the slowest unlearning method. 
We compare a variety of state-of-the-art FU schemes. First, we consider our method SIFU as described in Algorithm \ref{alg:SIFU} and set $B=1$. The choice for $B$ is experimentally justified in Appendix \ref{app:sec:B_justif}.
In addition to SIFU, we consider the following unlearning schemes from the state-of-the-art: \textsc{Scratch}, where retraining of a new model is performed from scratch on the remaining clients; \textsc{Fine-Tuning}, where retraining is performed on the current global model with the remaining clients; \textsc{DP} \cite{DP_book}, where training with every client is performed with Differential Privacy, and both \textsc{FedEraser} and \textsc{FedAccum} \cite{FedEraser}, where unlearning is performed by using the gradient history of clients to remove their contribution.
Finally, we consider a freely adapted version of \citet{DescentToDelete}'s perturbed gradient descent by noising the final model with a standard deviation calculated from SIFU's theoretical analysis, which boils down to performing SIFU without the "roll-back" step.
%retraining is performed on the current global model from which the server removes the updates of the clients to unlearn, by re-aggregating the parameter updates of clients that were stored by the  server across FL iterations. 
% We note that an optimal FU scheme should lead to a model with utility and unlearning capabilities as close as possible to the ones obtained with
% \textsc{Scratch}, while requiring fewer server aggregation rounds.

% The unlearning parameters considered in all our experiments are $\epsilon=10$ and $\delta = 0.01$.
\begin{figure*}[ht]
\begin{centering}
    \includegraphics[width=\linewidth]{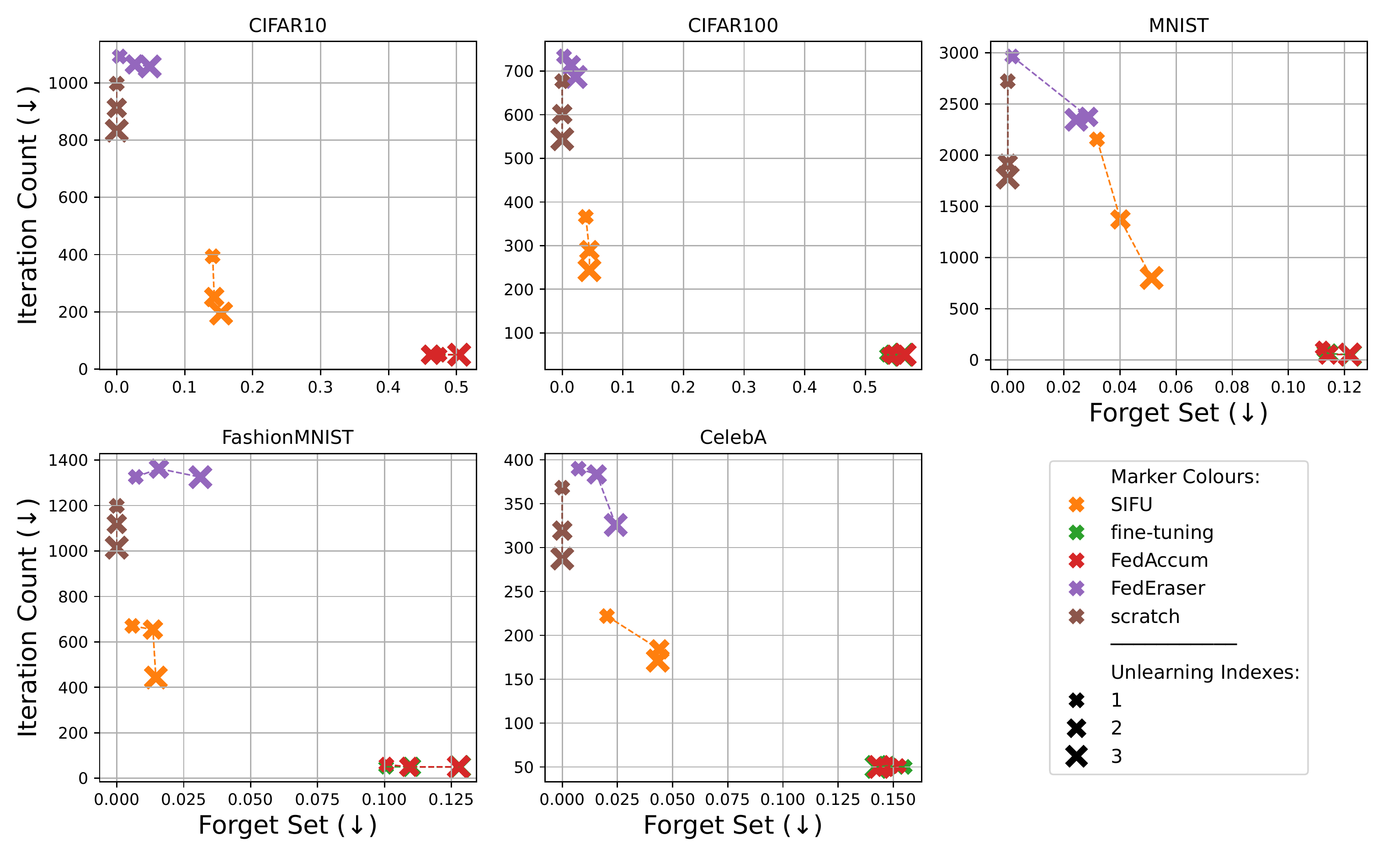}
\end{centering}
\caption{
Total amount of aggregation rounds (1\textsuperscript{st} row) and model accuracy of unlearned clients (2\textsuperscript{nd} row) for the unlearning of watermarked data from MNIST, FashionMNIST, CIFAR10, CIFAR100, and CelebA (the lower the better).
}
\label{fig:SIFU_backdoored}
\end{figure*}

\textbf{Experimental scenario.}

Since unlearning is about efficiency of information deletion, we first study the unlearning capabilities of every method under different time constraints.
We limit the unlearning time allowed for each method to respectively $25\%$, $33\%$ and $50\%$ of the training time, and display the accuracy results both on the unlearned clients (forget set) and the remaining ones (retain set).
Secondly, to account for the sequential unlearning setting and the adversarial case of watermarked data, we study the scenario proposed by \cite{TowardsProbabilisticVerification} in Section \ref{subsec:exp_watermark}.
 %The questions thus goes from ”how accurate is the model?” to ”how long does it take to get an accurate model?”
%The server orchestrates each unlearning scheme through retraining until the global model accuracy on the remaining clients exceeds a fixed threshold specific to each dataset.
Each unlearning method is applied with the same hyperparameters, i.e. local learning rate $\eta$, amount of SGD steps $K$ and optimizer (Appendix \ref{app:sec:experiments}).

% When unlearning and for any unlearning scheme, the server keeps retraining up until the model accuracy exceeds the same set value used for training but on the remaining clients, while performing a minimum of 50 server aggregations. Training and unlearning use the same hyperparameters, e.g. stopping accuracy value, local learning rate $\eta_l$, and amount of local work $K$, with values in Appendix \ref{app:sec:experiments}. At a unlearning request $r$, we unlearn 10 clients. In the special case of MNIST and FashionMNIST, we unlearn 10 clients with the same class. We unlearn thrice sequentially 10 clients to replicate a sequential federated unlearning setting. We define $F_r = \cup_{s=1}^r W_s$ the set of unlearned clients' indices.

\textbf{Unlearning quantification.}
We verify the success of a FU scheme by evaluating its running time and difference in accuracy with \textsc{Scratch} on the retain and forget set.
When studying methods under a time constraint (Section \ref{subsec:experimental_results}), we compute: (a) The accuracy difference on the forget set, thus evaluating unlearning quality, and (b) the accuracy difference on the retain set, thus evaluating retained utility. 
When studying methods under a utility constraint (Section \ref{subsec:exp_watermark}), we compute (a) The accuracy difference on the forget set, and (b) the number of required iterations to reach the fixed utility threshold.
The differences are computed with respect to the performance of \textsc{Scratch} to identify the method associated with similar unlearning properties and reduced computation time.

%We note that, by construction, \textsc{Scratch} provides perfect unlearning of the clients in a request $W_u$, and thus represents our FU baseline. %  Thus, we consider an unlearning scheme successful if it reaches an accuracy similar to \textsc{Scratch} on the forgotten clients ($F_u$), by requiring however less aggregation rounds.
% We note that an optimal FU scheme should lead to a model with utility and unlearning capabilities as close as possible to the ones obtained with \textsc{Scratch}, while requiring fewer server aggregation rounds.

\subsection{Experimental Results}
\label{subsec:experimental_results}

% \begin{figure}
% 	\begin{centering}
% 		\includegraphics[width=\linewidth]{./plots/fig_SIFU_small_backdoored.pdf}
% 	\end{centering}
% 	\caption{Total amount of aggregation rounds (1\textsuperscript{st} row) and model accuracy of unlearned clients (2\textsuperscript{nd} row) for the unlearning of watermarked data from CIFAR100 and CelebA. Results are reported with variability estimated on 10 seeds. 
%  Additional results for the pixel modified version of MNIST, CIFAR10, and CelebA in Appendix \ref{app:sec:experiments}.
%  }
% 	\label{fig:SIFU_small_backdoored}
%  
% \end{figure}

The results for the first experimental setting described in Sec.~\ref{subsec:experimental_setup} are available in Fig \ref{fig:NEW_SIFU}.
SIFU outperforms its competitors on 4 out of 5 datasets. In particular, on every experiment involving FU of a neural network, SIFU is the only method to achieve a good trade-off between forgetting quality and accuracy on the retain set, while other methods fail either by lack of forgetting quality (Fine-Tuning, FedAccum) or by low retain set accuracy (FedEraser, \textsc{Last}). 
% Finally, FedEraser slightly outperforms our method on the (convex) logistic regression task on MNIST.

It is interesting to notice that while \textsc{Last} performs quite well on very simple datasets such as CelebA and MNIST, the introduced noise becomes too large in more complex datasets such as CIFAR10(0), rendering it infeasible for the model to converge after noising. This motivates our method and underlines the importance of the "roll-back" step in SIFU.
While FedEraser tends to be a good performer in terms of retain set accuracy, the imposed unlearning time limit forbids it from recovering satisfying utility on the retain set, as compared to SIFU. Thus, one can wonder whether the method would perform better if provided with more time. We answer this question in Section \ref{subsec:exp_watermark}.
Finally, we excluded some methods when their poor performances would hinder the figure's readability: \textsc{DP} was removed from Figures \ref{fig:NEW_SIFU} and \ref{fig:SIFU_backdoored}, and \textsc{Last} from Figure \ref{fig:SIFU_backdoored}. See Appendix \ref{app:sec:experiments} for DP performances.

\subsection{Verifying Unlearning with Watermarking}
\label{subsec:exp_watermark}
The work of \cite{TowardsProbabilisticVerification} proposes an adversarial approach to verify the unlearning efficiency through watermarking. 
We follow this approach by applying watermarking to each client's data by randomly assigning the maximum possible value to 10 given pixels of each data sample. 
To ensure that clients' heterogeneity is only due to the modification of the pixels intensities, we define data partitioning across clients by randomly assigning the data according to a Dirichlet distribution with parameter $\alpha = 1$.

We consider a sequential unlearning scenario in which the server performs FL training and then receives $U=3$ sequential unlearning requests to unlearn 10 random clients per request. In the special case of MNIST and FashionMNIST, the server must unlearn 10 clients owning the same class.
The server performs unlearning with each method, before fine-tuning the obtained model on the remaining data until the global model accuracy on the remaining clients exceeds a fixed threshold specific to each dataset, namely:
 93\% for MNIST, 99.9\% for CelebA, and 90\% for FashionMNIST, CIFAR10 and CIFAR100. We impose a minimum of 50 FL aggregation rounds, and a maximum of 10000 rounds when the stopping accuracy threshold is not reached.

Figure \ref{fig:SIFU_backdoored} shows our results for this experimental scenario. On every dataset where a CNN is used, SIFU outperforms every other method, confirming the conclusions drawn from Figure \ref{fig:NEW_SIFU}. Indeed, it offers a unique trade-off between efficiency and unlearning, while FedEraser provides satisfying unlearning but is even more costly than \textsc{Scratch} in terms of iteration count, thus rendering the method of limited interest in our experimental scenario.

\subsection{Verifying the consistency of SIFU.}
\label{subsec:exp_nosie_std}

We provide additional experiments in Appendix \ref{app:sec:experiments} assessing the unlearning results depending on varying training conditions, including the choice of the clients to be unlearnt, and the amplitude of the model perturbation. Our findings are consistent in showing that SIFU is the best performing method in terms of computational efficiency and unlearning capabilities. In particular, we note that when unlearning with low (resp. high) values of $\sigma$, SIFU has identical behavior to \textsc{Scratch} (resp. \textsc{Last}), as the unlearning is applied to the initial model $\vtheta_0^0$ (resp. final $\vtheta_u^{N_u}$). Moreover, independently from the chosen batch of clients to be unlearnt, Supplementary Figure \ref{fig:forgetting_many_class} shows that SIFU consistently leads to effective unlearning with lower computational cost as compared to \textsc{Scratch}.

\section{Conclusions}
In this work, we introduce SIFU, a general FU scheme allowing unlearning of clients contributions from a model trained with \textsc{FedAvg}. Upon receiving an unlearning request from a given client, SIFU identifies the optimal FL iteration from which to re-initialise the optimisation. We prove that SIFU accounts for sequences of requests while satisfying the unlearning guarantees. SIFU is scalable with respect to model size and FL iterations, and generalizes beyond the convex assumption on the local
loss functions, thus relaxing the strong assumptions typically adopted in the MU literature.

A further contribution of this work consists in a new theory for bounding the clients contribution in FL, which can be computed by the server without major overhead, and no additional communication nor computation on the client side. %.: strong convexity of the clients' loss functions, approximation of the clients' gradients, and computation or approximation of the client's Hessian.
% it provides bounds for the smooth, convex and strongly convex cases and is thus applicable to a wide range a settings, without any additional computation required from the client.

%The relevance of our approach is also experimentally verified on both convex and non-convex problems across several benchmarks.

% SIFU is the first approach to achieve FU with unlearning guarantees, without assuming strong convexity. It also avoids relying on any approximation of the gradient and does not require the communication of the client's Hessian (or any of its estimates). Additionally, it provides bounds for the smooth, convex and strongly convex cases and is thus applicable to a wide range a settings, without any additional computation required from the client. The relevance of our approach is also verified on both convex and non-convex problems across several benchmarks.

% An additional contribution of this work consists in a new theory for bounding the clients contribution in FL, which can be computed for every client without additional computation and communication. The theoretical justification of SIFU only requires clients local loss functions to be Lipschitz smooth, which generalizes the typical quadratic approximation  of the FU literature. The relevance of our theory is demonstrated experimentally on both convex and non-convex problems across several benchmarks.

\section{Acknowledgements}
This work was supported by the French government managed by the Agence Nationale de la Recherche (ANR) through France 2030 program with the reference ANR-23-PEIA-005 (REDEEM project), and through the Franco-German research program with reference ANR-22-FAI1-0003 (TRAIN project). It was also funded in part by the Groupe La Poste, sponsor of the Inria Foundation, in the framework of the FedMalin Inria Challenge.

% \clearpage

\bibliography{main}

\clearpage
\appendix

\section{When fine tuning does not guarantee unlearning: example on linear regression}\label{app:sec:initialization}
Let us consider a linear regression optimization, with feature matrix $\mX$ and predictions $\vy$ such that the loss function $f$ is defined as
\begin{equation}
f(\mX, \vy, \vtheta) = \frac{1}{2} \left[\vy - \mX \vtheta \right]^T \left[\vy - \mX \vtheta \right]
 \label{eq:def_f_example}
.
\end{equation}
In this example, we assume there are more features than data samples, which makes $\mX^T \mX$ a singular matrix. While $f$ is convex, $f$ has more than one global optimum. Any model with parameter $\vtheta^*$ such that 
\begin{equation}
    \mX^T \mX \vtheta^* 
    = \mX^T \vy
\end{equation}
is a global optimum. When $\mX^T \mX$ is non-singular, we retrieve the unique optimum in close-form $\vtheta^* = \left(\mX^T \mX\right)^{-1}\mX^T \vy$. 
We show with this simple example that, upon unlearning a data sample, no amount of fine-tuning on the model $\vtheta^*$ can lead to the same model obtained when retraining from a random initial model. 
We differentiate between $(\mX, \vy)$ and $(\mX_{-1}, \vy_{-1})$ our data with and without a given data point. 

Optimizing $f$, as defined in equation (\ref{eq:def_f_example}), with $N$ steps of gradient descent, learning rate $\eta$, and initial model $\vtheta_0$ gives model parameters $\vtheta^N$ defined as
\begin{equation}
\vtheta^N 
= \underbrace{\left[I - \eta \mX^T \mX\right]^N}_{A(\mX, N)} \vtheta^0 
+ \underbrace{\eta \sum_{n =0}^{N-1} \left[I - \eta \mX^T \mX\right]^n \mX^T \vy}_{B(\mX, \vy, N)}
.
\end{equation}

We first note that we retrieve the standard form for the global optimum of linear regression when $\mX^T \mX$ is non-singular as $\lim_{n\to\infty}A(\mX, n) = 0$ and $\lim_{n\to\infty}B(\mX, \vy, n) = \left(\mX^T \mX\right)^{-1}\mX^T \vy$. In the general form accounting for the singular case, at least one eigenvalue of $A(\mX, N)$ is equal to 1 independently from the amount of gradient descent steps $N$. Hence, the parameters of the model obtained with gradient descent optimization always depend from the ones of the initial model $\vtheta^0$. Hence, when unlearning our data sample from $\vtheta^N$, the resulting trained model still depends of that data samples. 
Indeed, if we compare the model $\vtheta_{-1}^{\tilde{N}}$ trained on the data samples $(\mX_{-1}, \vy_{-1})$,  to  the model  $\phi_{-1}^{\tilde{N}}$ obtained after fine-tuning the model $\vtheta^N$ with $\tilde{N}$ server aggregations, we have
\begin{equation}
    \phi_{-1}^{\tilde{N}} - \vtheta_{-1}^{\tilde{N}}
    = A(\mX_{-1}, \tilde{N}) A(\mX, N) \vtheta^0 
+ A(\mX_{-1}, \tilde{N}) B(\mX,  \vy, N)
.
\end{equation}

\section{Forgetting a Single Client with IFU, Theorem \ref{theo:diff_bound}}
\label{app:sec:proof_diff}
In this section, we provide the proof of Theorem \ref{theo:diff_bound} and derive 3 different results when considering 3 different sets of assumptions: $f$ is smooth, $f$ is smooth and convex, and $f$ is smooth and strongly-convex.

\subsection{Definitions}
We define by $\vtheta^N = \textsc{FedAvg}(I, N)$ and $\vphi^N = \textsc{FedAvg}(I_{-c}, N)$ the models trained with $\textsc{FedAvg}$ initialized at $\vtheta_0$ with respectively all the clients, i.e. $I$, and all the clients but client $c$, i.e. $I_{-c}$, performing $K$ GD steps.

    When clients perform $K=1$ GD steps, two consecutive global models can be related, when training with clients in $I$ as a simple GD step, i.e.
    \begin{equation}
	\vtheta^{n+1}
	= \vtheta^n
	- \eta \nabla f_I(\vtheta^n)
	.
    \end{equation}

Let us define the gradient step operator for function $f$ at learning rate $\eta$:
$$
G(f, \eta, \vtheta) = \vtheta - \eta \nabla f(\vtheta)
$$

\subsection{General case}

\subsubsection{Main observation}

The following results use Lemma 3.7 of \cite{hardt2016} under its 3 possible hypothesis. Let us first notice that, with $K=1$ and without any hypothesis on $f$ besides its differentiability, we have:

\begin{align*}
\begin{split}
\phi^{n+1} - \theta^{n+1}  
&= \phi^n - \theta^n - \eta[\nabla f_{I\setminus \{c\}}(\phi^n) \\
&\hphantom{aa} - \nabla f_{I\setminus \{c\}}(\theta^n) + \nabla f_{I\setminus \{c\}}(\theta^n) - \nabla f_I (\theta^n)] \\
&= G(f_{I\setminus \{c\}}, \eta, \phi^n) - G(f_{I\setminus \{c\}}, \eta, \theta^n) \\
&\hphantom{aa} + \eta (\nabla f_{I\setminus \{c\}}(\theta^n) - \nabla f_I (\theta^n))  \\
\end{split}
\end{align*}

Then, depending on the assumptions made on $f$,  we get 3 different results, all taking the same form:
\begin{align}
\begin{split}
    \lVert \phi^{n+1} - \theta^{n+1} \rVert 
    &\le B(f, \eta)\lVert \phi^n - \theta^n \rVert \\
    &+ \eta \lVert \nabla f_{I\setminus \{c\}}(\theta^n) - \nabla f_I (\theta^n) \rVert \\
\end{split}
\end{align}
Where we consider 3 distinct cases, each with their respective assumptions and definition of $B$:
\begin{enumerate}
    \item If $f_i$ is $\beta$-smooth for every $i \in I$, then
    \begin{equation}\label{eq:B1}
        B(f_I, \eta) = 1 + \eta.\beta
    \end{equation}
    \item If $f_i$ is $\beta$-smooth and convex for every $i \in I$ and $\eta \le 2 / \beta$, then
    \begin{equation}\label{eq:B2}
        B(f, \eta) = 1
    \end{equation}
    \item If $f_i$ is $\beta$-smooth and $\mu$-strongly-convex for every $i \in I$ and $\eta \le \frac{2}{\beta + \mu}$, then
    \begin{equation}\label{eq:B3}
        B(f, \eta) = 1 - \frac{\eta\beta\mu}{\beta+\mu}
    \end{equation}
\end{enumerate}

\subsubsection{Generic proof}

Let us prove the desired results with a generic function $B$. The specific results in the 3 different cases will then be derived directly by specifying $B$ depending on the hypothesis.

Let $p_i = \frac{N_i}{N_I}$ and $q_i = \frac{p_i}{1-p_c}\cdot(1-\mathbbm{1}_{c}(i))$. Then,

\begin{align*}
    \lVert \phi^{n+1} - \theta^{n+1} \rVert
    &= \Big\lVert \sum\limits_{i=1}^{M} q_i\phi_i^{n+1} - \sum\limits_{i=1}^{M} p_i\theta_i^{n+1} \Big\rVert \\
    &= \Big\lVert \sum\limits_{i=1}^{M} q_i(\phi_i^{n+1}-\theta_i^{n+1})\\ 
    &\hphantom{aa} + \underbrace{\sum\limits_{i=1}^{M} q_i(\theta_i^{n+1}-\theta^{n}) - \sum\limits_{i=1}^{M} p_i(\theta_i^{n+1} - \theta^n)}_{\Delta_c (I, \theta^n)} \\
    &\le \sum\limits_{i=1}^M q_i\lVert \phi_i^{n+1}-\theta_i^{n+1}\rVert + \Delta_c (I, \theta^n) \\
    &\le \max_{i \in I} \lVert \phi_i^{n+1}-\theta_i^{n+1}\rVert + \Delta_c (I, \theta^n)
\end{align*}

where the last inequality follows from the fact that $\sum q_i = 1$.

Now, for any $i \in I$, let us give an upper bound for $\lVert \phi_i^{n+1}-\theta_i^{n+1}\rVert$:

\begin{equation}\label{eq:boundB}
\begin{aligned}
    \lVert \phi_i^{n, k+1} - \theta_i^{n, k+1} \rVert &= \lVert G(f_{\{i\}}, \eta, \phi_i^{n, k}) - G(f_{\{i\}}, \eta, \theta_i^{n, k}) \rVert\\
    &\le  B(f, \eta)\cdot\lVert \phi_i^{n, k} - \theta_i^{n, k} \rVert
\end{aligned}
\end{equation}
where the last inequality follows from the contractivity of one step of gradient descent (see Lemma 3.7 of \citet{hardt2016} for all 3 cases).
By applying this equation recursively $K$ times, we get 
\begin{equation}\label{eq:boundB2}
\lVert \phi_i^{n+1} - \theta_i^{n+1} \rVert \le  B(f, \eta)^K \cdot \lVert \phi_i^{n} - \theta_i^{n} \rVert    
\end{equation}

 From \Eq{boundB} and \Eq{boundB2}, we get:
 \begin{equation}
     \lVert \phi^{n+1} - \theta^{n+1} \rVert \le B(f, \eta)^K \cdot\lVert \phi^{n} - \theta^{n} \rVert + \Delta_c(I, \theta^n)
 \end{equation}

Now, let us prove via recurrence that:
\begin{equation}\label{eq:recurrence}
    \lVert \phi^{n} - \theta^{n} \rVert \le \sum\limits_{p=0}^{n-1} B(f, \eta)^{(n-p-1)K}\cdot \Delta_c(I, \theta^p)
\end{equation}

The initialization is trivial since $\phi^0 = \theta^0$.

We provide the proof of the recurring property in equation (\ref{eq:proof_req}), thus verifying equation (\ref{eq:recurrence}) and proving Theorem \ref{theo:diff_bound}.
\begin{equation}\label{eq:proof_req}
\begin{aligned}
    \lVert \phi^{n+1} - \theta^{n+1} \rVert &\le B(f, \eta)^K \cdot \\
    &\Biggr[\sum\limits_{p=0}^{n-1} B(f, \eta)^{(n-p-1)K} \Delta_c(I, \theta^p)\Biggr]\\
    &\quad + \Delta_c (I, \theta^n) \\
    &\le \sum\limits_{p=0}^{(n+1)-1} B(f, \eta)^{(n+1-p-1)K}\cdot \Delta_c(I, \theta^p)
\end{aligned}
\end{equation}

Let us now give the specific formulations for each set of assumptions.

\subsection{Case $f$ smooth, not necessarily convex}

Let $f$ be $\beta$-smooth. In this case, the result from \Eq{recurrence} applies with $B$ as described in \Eq{B1}. Therefore:
\begin{equation}
\lVert \phi^{n} - \theta^{n} \rVert \le \sum\limits_{p=0}^{n-1} (1+\eta \beta)^{(n-p-1)K}\cdot \Delta_c(I, \theta^p)
\end{equation}

\subsection{Case $f$ smooth $\And$ convex}
Let $f$ be convex and $\beta$-smooth, let $\eta \le \frac{2}{\beta}$. In this case, the result from \Eq{recurrence} applies with $B$ as described in \Eq{B2}. Therefore:
\begin{equation}
\lVert \phi^{n} - \theta^{n} \rVert \le \sum\limits_{p=0}^{n-1}  \Delta_c(I, \theta^p)
\end{equation}

\subsection{Case $f$ smooth $\And$ strongly convex}
Let $f$ be $\mu$-strongly convex and $\beta$-smooth, let $\eta \le \frac{2}{\beta + \mu}$. In this case, the result from \Eq{recurrence} applies with $B$ as described in \Eq{B3}. Therefore:
\begin{equation}
\lVert \phi^{n} - \theta^{n} \rVert \le \sum\limits_{p=0}^{n-1} (1 - \frac{\eta\beta\mu}{\beta+\mu})^{(n-p-1)K}\cdot \Delta_c(I, \theta^p) 
\end{equation}

\section{Unlearning certification, Proof of Theorem \ref{theo:noise_DP}}
\label{app:sec:proof_DP}

\begin{proof}
According to theorem A.1 of \cite{DP_book}, the Gaussian mechanism with sensitivity $\alpha$, is $(\epsilon, \delta)$-Differentially Private if its noise parameter $\sigma$ verifies :

    \begin{equation}
        \sigma
        > \left[2 \left( \ln(1.25) - \ln(\delta) \right) \right]^{1/2} \epsilon^{-1} \alpha.
    \end{equation}

In our case, the sensitivity with respect to any client $c$ at step $n$ is $\alpha(n, c)$ by construction. Additionally, Theorem \ref{theo:diff_bound} provides us with the bound 

\begin{equation}
    \alpha(n, c) \le \Psi(n, c).
\end{equation}

Thus, picking any noise $\sigma(n, c)$ such that

    \begin{equation}
        \sigma(n, c)
        > \left[2 \left( \ln(1.25) - \ln(\delta) \right) \right]^{1/2} \epsilon^{-1} \Psi(n, c)
    \end{equation}
ensures the differential privacy of our forgetting mechanism with respect to the unlearned clients. This concludes the proof.
\end{proof}
\section{Convergence of SIFU, Theorem \ref{theo:zeta}}
\label{app:sec:SIFU_convergence}

\begin{figure*}[ht]
\begin{centering}
    \includegraphics[width=\linewidth]{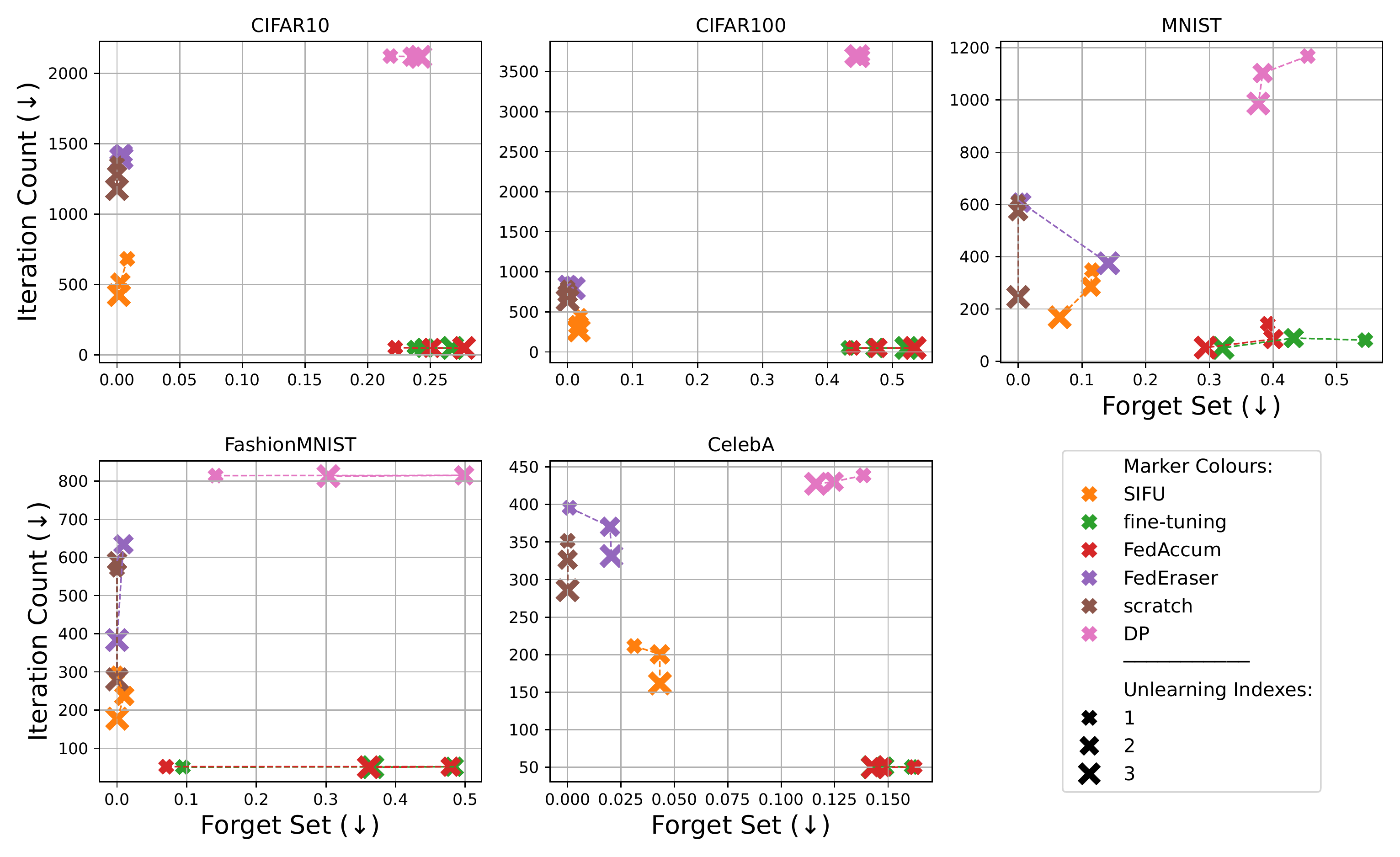}
\end{centering}
\caption{
Total amount of aggregation rounds (1\textsuperscript{st} row) and model accuracy of unlearned clients (2\textsuperscript{nd} row) for MNIST, FashionMNIST, CIFAR10, CIFAR100, and CelebA (the lower the better). 
The server runs a federated routine with $M=100$ clients, and unlearns 10 of them at each unlearning request ($U=3$). 
Results are reported with variability estimated on 10 seeds.
}
	\label{fig:SIFU}
\end{figure*}

% \subsection{Intermediate results}

% \begin{property}\label{app:prop:SIFU_increasing_t}
%     If there exists $\nu$, $s$, $u$ such that $s < u$, $(\nu, t_s) \in O(s)$ and $(\nu, t_u) \in O(u)$, then $t_s \ge t_u$.
% \end{property}

% \begin{proof}
%     % Let us consider $\nu$, $s$ and $u$ such that $u = s + 1$, $(\nu, t_s) \in O(s)$ and $(\nu, t_u) \in O(u)$.
%     We first assume that $s$ and $u$ satisfy $u = s + 1$.
%     Considering that $(\nu, t_s) \in O(s)$ and $(\nu, t_u) \in O(u)$, we have, by definition of $\zeta_u$ in equation (\ref{eq:zeta_r}), $\nu \le \zeta_u$. 
%     \begin{itemize}
    
%         \item $\zeta_u>\nu$. Considering that $u = s + 1$, we have $t_s = t_u$, equation (\ref{eq:oracle_recurrent}). 
        
%         \item $\zeta_u = \nu$. Considering that $(\nu, t_s) \in O(s)$ and $(\nu, t_u) \in O(u)$, then we have $\nu \le s-1$. Therefore, by definition of $\zeta_u$, we have $\Psi_{\zeta_u}(t_s, W_u) > \Psi^*$. By construction of $T_u$, equation (\ref{eq:T_SIFU}), we have $t_u = T_u < t_s$.        
%     \end{itemize}
%     When considering the more general case where there exists an integer $k$ such that $u = s +k$ while $(\nu, t_s) \in O(s)$ and $(\nu, t_u) \in O(u)$, then it is sufficient to consider iteratively an integer $j$ ranging from 1 to $k$. Considering $(\nu, t_u) \in O(u)$, there exists $t_{s+j}$ such that $(\nu, t_{s+j}) \in O(s+j)$. In that case, using the same reasoning as for $k=1$, we have  $t_s \le t_{s+1} \le \ldots \le t_{s+k-1} \le t_u$. 

% \end{proof}

\subsection{Proof of Theorem \ref{theo:zeta}}

\begin{proof}
%    Proving that $\vtheta_r^{N_r}$ $(\epsilon, \delta)$-unlearns every client in $F_r$, equation (\ref{eq:F_r}), reduces to proving that $\vtheta_r^0$ $(\epsilon, \delta)$-unlearns every client in $F_r$, equation (\ref{eq:F_r}). Indeed, the data of clients in $F_r$ are not used on the noised perturbed model $\vtheta_r^0 = \vtheta_{\zeta_r}^{T_r} + \Ncal(0, \sigma^2 \mI_\vtheta)$.

    We prove by induction that $\vtheta_{u+1}^0$ $(\epsilon, \delta)$-unlearns every client in $F_{u+1} = \cup_{s=1}^{u+1} W_s = F_u \cup W_{u+1}$.
    The initialization ($u=1$) directly follows from IFU, Algorithm \ref{alg:unlearning_ours}, with Theorem \ref{theo:noise_DP}. With reference to equation (\ref{eq:theta_T_zeta}), we now assume that for every unlearning request $u\le u$, the perturbed model $\vtheta_u^0$ $(\epsilon, \delta)$-unlearns every client in $F_u$, and prove that $\vtheta_{u+1}^0$ $(\epsilon, \delta)$-unlearns every client in $F_{u+1}$.
    
    \begin{itemize}
        \item  Case 1: {$\forall u \leq u$, $n_u \leq n_{u+1} $}.    
        %$\vtheta_{\zeta_r}^0$ $(\epsilon, \delta)$-unlearns every clients in $W_s$. Clients in $W_s$ are not used for training on $\vtheta_{\zeta_r}^0$. 
        The model $\vtheta_{\zeta_{u+1}}^{T_{u+1}} = H(u)[n_{u+1}]$ appears later in the training history than models $\vtheta_{\zeta_{u}}^{T_{u}}$ and, thanks to the induction property,  provides $(\epsilon, \delta)$-unlearning of every client in $F_u$. Thus, the model  $\vtheta_{u+1}^{0}$ guarantees the unlearning of every client in $F_{u+1}$. 

        \item Case 2: $\exists\text{ unlearning request } u^* \leq u \text{ such that } n_{u+1} < n_{u*}$. 
        %According to equation (\ref{n_u+1}), this there exists a learning request $u^*$ such that, for every client $c^*\in W_{u^*}$ , we have:
        % $$ \Psi_{u^*}(n_{u^*},c^*) \leq \Psi_{u+1}(n_{u+1},c) \leq \Psi^* $$ for at least a client $c\in W_{u+1}$ (equation \ref{eq:def_Psi_extended_set}). 
        % We also note that Case 2 also implies $\Psi_{u^*}(n_{r+1},c^*)\leq \Psi_{u^*}(n_{u^*},c^*)$, 
        By construction of the training history, the sequence $H(u^*)$ contains the model $\vtheta_{\zeta_{u+1}}^{T_{u+1}} = H(u^*)[n_{u+1}] = H(u)[n_{u+1}]$, which appears earlier than model $\vtheta_{\zeta_{u^*}}^{T_{u*}} = H(u^*)[n_{u^*}]$. Perturbing the model $\vtheta_{\zeta_{u+1}}^{T_{u+1}}$ with noise  $N(\mathbf{0},\sigma\mI_\vtheta$), guarantees $(\epsilon, \delta)$-unlearning of the clients in $W_{u^*}$, since  $$ \Psi_{u^*}(n_{u+1},c^*) \leq \Psi_{u^*}(n_{u^*},c^*) \leq \Psi^*, $$ 
        for every client $c^*$ in $W_{u*}$. By extending this reasoning to all learning requests $u$ such that $n_{u+1} < n_{u}$, and by the induction property for the remaining ones, the model  $\vtheta_{u+1}^{0}$ guarantees the unlearning of every client in $F_{u+1}$.  

\end{itemize}

\end{proof}

\section{Experiments}
\label{app:sec:experiments}

For every benchmark, we consider the number of SGD steps $K$, batch size $B$, number of clients $M$, the number of sampled clients $m$, the standard deviation $\sigma$ of the noise perturbation, and the local learning rate $\eta$ given in Table \ref{table:hyperparameters}. 
Also, for our unlearning scheme \textsc{SIFU}, \textsc{DP}, and \textsc{Last}, we consider an unlearning budget of $\epsilon=10$ and $\delta = 0.01$. The unlearning budget plays the important role of identifying in the training history the global model to perturb. Theorem \ref{theo:noise_DP} shows that $\epsilon$ and $\sigma$ are linearly related. Hence, to unlearn a client $c$ from a global model $c$, a smaller $\sigma$ can be considered, but at the cost of a higher unlearning budget $(\epsilon, \delta)$, Definition \ref{def:DP_adpated}.
Also, for fair comparison of DP with other FU schemes, we select the best clipping value $C$, in a range from 0.001 to 1, for which the global model reaches the target accuracy in the smallest amount of aggregation rounds.
Finally, for FashionMNIST, CIFAR10, CIFAR100, and CelebA, we consider model architectures composed of three convolutional layers followed by two fully connected layers, with implementation at \href{https://github.com/Accenture/Labs-Federated-Learning/tree/SIFU}{this GitHub url}.

\begin{table}[htb]
\caption{Hyperparameters used for our different unlearning benchmarks described in Section \ref{subsec:experimental_setup}.} \label{table:hyperparameters}
\begin{center}
\begin{tabular}{|l|c|c|c|c|c|c|c|}
    % \toprule
    \hline
    \textbf{Dataset} & $K$ & $B$ & $M$ & $m$ & $\sigma$ & $\eta$ & $C$
    \\\hline\hline
    % \midrule

    CIFAR10 & 5 &  20 & 100 & 5 & 0.05 & 0.01 & 0.2
    \\ \hline
    CIFAR100 & 5 & 20 & 100 & 5 & 0.05 & 0.02 & 0.2
    \\ \hline
    MNIST &  10 & 100 & 100 &  10 &  0.05 & 0.01 & 0.5
    \\ \hline
    FMNIST & 5 & 20 & 100 & 10 & 0.1 & 0.02 & 0.5
        \\  \hline
    CelebA & 10 & 20 & 100 & 20 & 0.1 & 0.01  & 0.5

    \\ \hline
    % \bottomrule
\end{tabular}
\end{center}
\end{table}

The training and retraining depends on the initial model $\vtheta_0^0$ and the clients' batches of data used at every aggregation to compute their local SGDs. Hence, we replicate each unlearning scenario on 10 different seeds and plot in Figure \ref{fig:SIFU} to \ref{fig:forgetting_many_class} their averaged results. For the unlearning benchmarks described in Section \ref{subsec:experimental_setup} and used in Figure \ref{fig:SIFU}, to \ref{fig:forgetting_many_class}, the stopping accuracies considered are 93\% for MNIST, 90\% for FashionMNIST, CIFAR10, and CIFAR100, and 99.9\% for CelebA. 
% For Figure \ref{fig:SIFU_small_backdoored} and \ref{fig:SIFU_backdoored} with unlearning benchmark described in Section \ref{subsec:exp_watermark}, the stopping accuracies considered are instead 99.9\% for MNIST, FashionMNIST, CIFAR10, and  CelebA, and 99\% for CIFAR100. Reaching such accuracies is easier with the backdoored datasets because the clients' data heterogeneity is only due to their watermark, Section \ref{subsec:exp_watermark}.

We provide several figures to further the experimental evaluation of our method.

 We define the set of clients requesting unlearning as:
\begin{equation}
    F_u 
    = \cup_{s=1}^u W_s
    .
    \label{eq:F_u}
\end{equation}

In our experimental scenario, we have $|F_0|=0$ during training, and $|F_1|=10$, $|F_2|=20$, and $|F_3|=30$ after each unlearning request.
We consider this setting both within an usual and an adversarial scenario with backdoored data (as proposed in \cite{TowardsProbabilisticVerification}).
In Figure \ref{fig:SIFU_backdoored} and \ref{fig:SIFU}, we compare the performances of SIFU with other methods from the literature. Rather than operating within a limited time budget, we fine-tune after the unlearning until a certain accuracy threshold is reached. Thus, all compared methods have equivalent performances on the retain set and the evaluated quantities are now forget set accuracy and unlearning (+ fine-tuning) budge, as described in Sec \ref{subsec:experimental_setup}.
Figure \ref{fig:SIFU} illustrates the computational cost (1st row), and unlearning capabilities (2ns row) of the tested FU methods across dataset. We note that SIFU requires a sensibly lower number of iterations than \textsc{Scratch} (52\% faster on average) to achieve similar unlearning performances. \textsc{FedEraser} also provides comparable unlearning capabilities, while however requiring a higher number of iterations than \textsc{Scratch} (12\% on average). The other approaches are generally associated with poor unlearning results, independently from the required computational cost. 
We notice that the model accuracy of SIFU is slightly higher than the one of \textsc{Scratch}, with overlap only for FashionMNIST. This behavior is natural and can be explained by the privacy budget $(\epsilon, \delta)$ trading unlearning capabilities for retraining cost. With the highest unlearning budget, i.e. $\epsilon =1$ and $\delta=0$, SIFU would require to retrain from the initial model $\vtheta_0^0$, thus reducing to \textsc{Scratch}.

The poor unlearning performance of \textsc{DP} can be explained by the fact that it provides privacy guarantees with respect to every client, while FU only aims at removing the contribution of a few specific clients. 

As observed previously, only SIFU and \textsc{FedEraser} provide satisfactory unlearning, but SIFU is significantly faster than its counterpart: \textsc{FedEraser} is even slower than \textsc{Scratch} and is thus not a relevant unlearning method in our experimental scenario.

Finally, when unlearning with \textsc{Last}, we observed that the model always converged to a local optimum with accuracy inferior to our target. This behavior is likely due to the magnitude of the noise being added to guarantee unlearning. Hence, we decided to exclude \textsc{Last} from Figure \ref{fig:SIFU}.

 Experiments were performed using a 1080TI GPU from Nvidia.

\begin{figure*}[htb]
\begin{centering}
    \includegraphics[width=\linewidth]{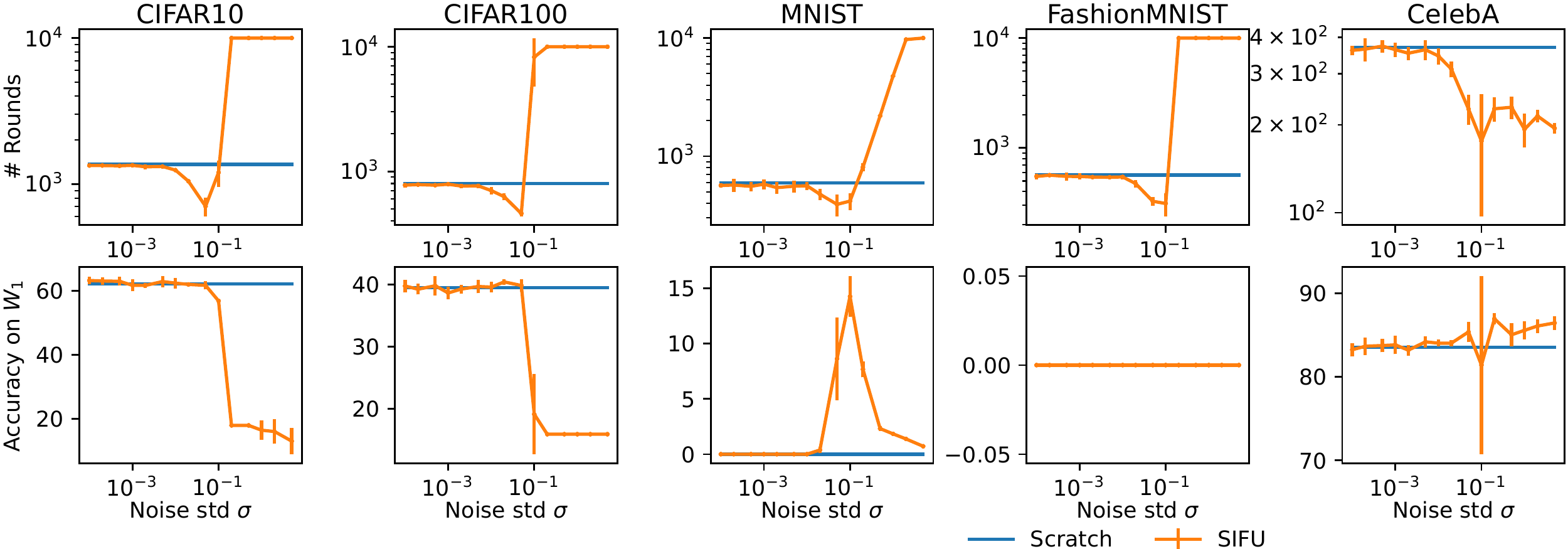}
\end{centering}
\caption{Impact of the noise standard deviation $\sigma$ when unlearning with SIFU for the unlearning budget $(\epsilon, \delta) = (10, 0.01)$. Total amount of aggregation rounds (1\textsuperscript{st} row) and model accuracy of unlearned clients (2\textsuperscript{nd} row) for MNIST, FashionMNIST, CIFAR10, CIFAR100, and CelebA (the lower the better). Speed-ups at optimal sigma are between two-fold and five-fold.
}
\label{fig:impact_std}
\end{figure*}

\begin{figure*}[ht]
\begin{centering}
\includegraphics[width=\linewidth]{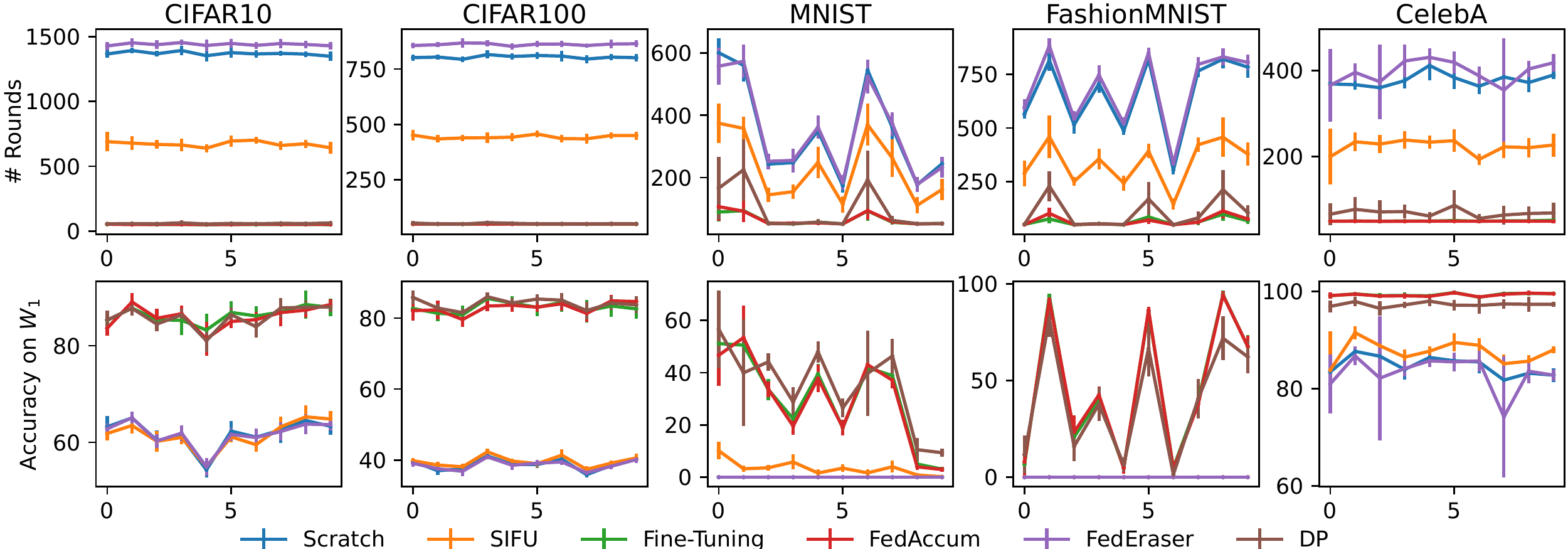}
\end{centering}
\caption{
Total amount of aggregation rounds (1\textsuperscript{st} row) and model accuracy of unlearned clients (2\textsuperscript{nd} row) for MNIST, FashionMNIST, CIFAR10, CIFAR100, and CelebA (the lower the better).
This figure displays the unlearning capabilities of the unlearning benchmarks introduced in Section \ref{subsec:experimental_setup} after training on clients in $I$ and unlearning $|W_1| = 10$ clients. For each integer on the x-axis, a different set of clients to unlearn is considered. 
Each unlearning request is composed of 10 random clients for CIFAR10, CIFAR100, and CelebA. For MNIST and FashionMNIST, each unlearning request $|W_1|$ has 10 clients of the same class such that the x-axis is the class forgotten. The integers on the x-axis corresponds to the class of the clients to unlearn. We retrieve the same conclusions made in \ref{fig:SIFU}: SIFU is the only unlearning method offering an unlearning speed-up while maintaining an accuracy close to \textsc{Scratch} on unlearned clients. DP seems very fast since the displayed number of rounds does not include the initial training.
}
\label{fig:forgetting_many_class}
\end{figure*}

\label{app:sec:FedAccum}

\section{Bound from theorem \ref{theo:diff_bound}}
\label{app:sec:B_justif}

To investigate whether it is legitimate to pick $B(f, \eta) = 1$ for our experiments, we empirically measure $\alpha$ and $\Psi$ for all our datasets with a large set of hyper-parameter choices, ranging from the ones used in the experiments to less adapted ones, even included some that do not allow convergence. The bound from Theorem \ref{theo:diff_bound} holds for every experimental scenario we tried. The results are displayed in figure \ref{app:fig:alpha_psi}.

\begin{figure*}
    \centering
       \
        \includegraphics[width=0.8\textwidth]{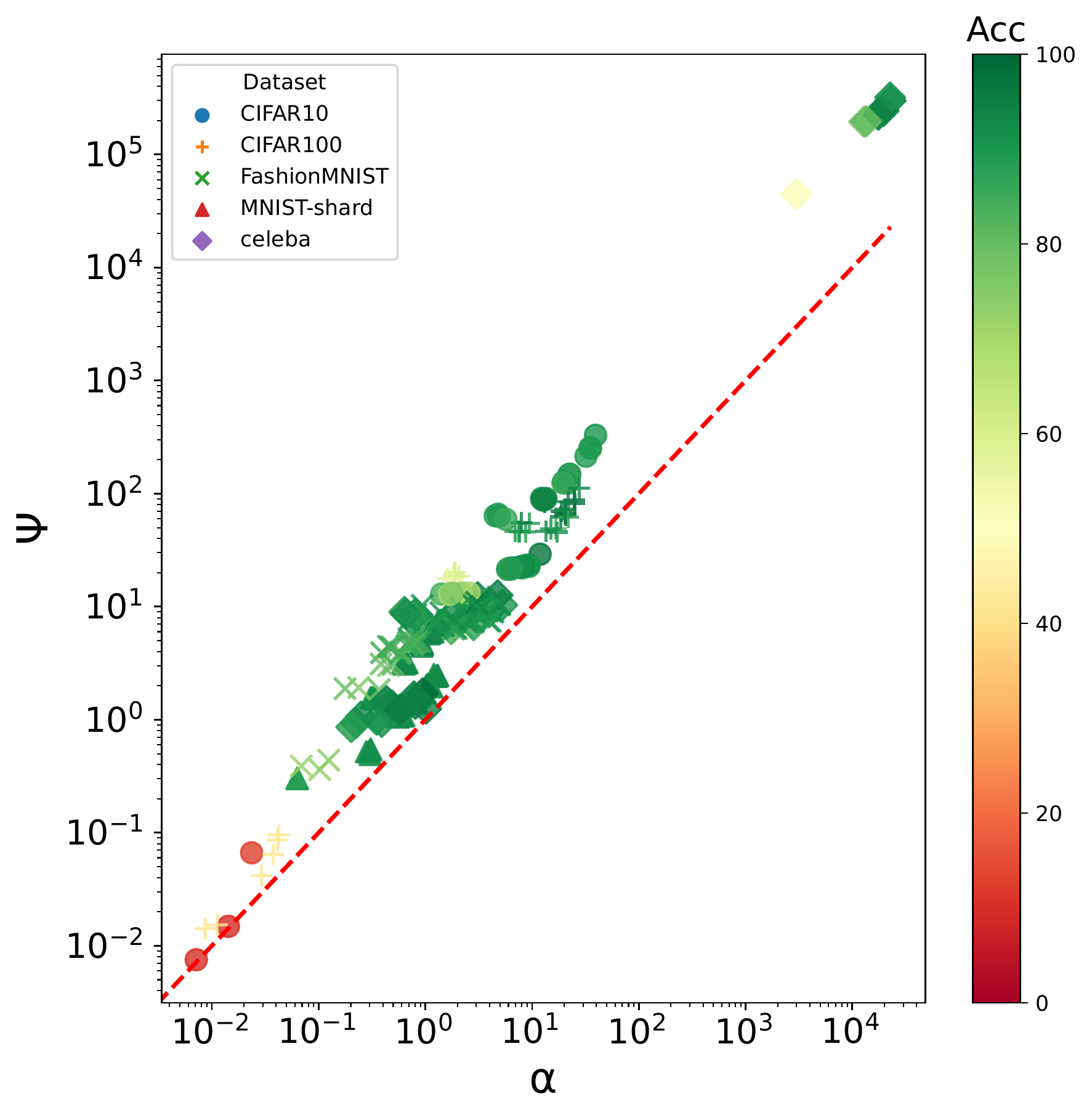}

    \caption{Comparison between $\Psi$ and $\alpha$ for $B=1$ for all datasets and various hyper-parameters. We observe that the bound demonstrated in Theorem \ref{theo:diff_bound} holds for all the considered experimental scenarios, even when setting $B=1$ in neural networks.}
\label{app:fig:alpha_psi}
\end{figure*}

\newpage

\end{document}